\documentclass[final]{svjour3}

\usepackage[pdftex]{graphicx}
\usepackage[cmex10]{amsmath}
\usepackage{amssymb,url,xspace,caption}
\usepackage{booktabs}
\usepackage[ruled,vlined]{algorithm2e}
\usepackage{algorithmic}
\usepackage{array,balance}
\usepackage{fixltx2e}
\usepackage{times}
\usepackage{subcaption}
\usepackage{pdfsync}

\usepackage{hyperref}  
\hypersetup{backref=true,colorlinks=true,linkcolor=red,citecolor=blue
,urlcolor=blue} 
\usepackage{etoolbox}
\makeatletter
\patchcmd{\NAT@citex}
  {\@citea\NAT@hyper@{\NAT@nmfmt{\NAT@nm}\NAT@date}}
  {\@citea\NAT@nmfmt{\NAT@nm}\NAT@hyper@{\NAT@date}}
  {}
  {}

\patchcmd{\NAT@citex}
  {\@citea\NAT@hyper@{%
     \NAT@nmfmt{\NAT@nm}%
     \hyper@natlinkbreak{\NAT@aysep\NAT@spacechar}{\@citeb\@extra@b@citeb}%
     \NAT@date}}
  {\@citea\NAT@nmfmt{\NAT@nm}%
   \NAT@aysep\NAT@spacechar%
   \NAT@hyper@{\NAT@date}}
  {}
  {}

\patchcmd{\NAT@citex}
  {\@citea\NAT@hyper@{%
     \NAT@nmfmt{\NAT@nm}%
     \hyper@natlinkbreak{\NAT@spacechar\NAT@@open\if*#1*\else#1
\NAT@spacechar\fi}%
       {\@citeb\@extra@b@citeb}%
     \NAT@date}}
  {\@citea\NAT@nmfmt{\NAT@nm}%
   \NAT@spacechar\NAT@@open\if*#1*\else#1\NAT@spacechar\fi%
   \NAT@hyper@{\NAT@date}}
  {}
  {}
\makeatother

\graphicspath{{./}{./figures/}}
\DeclareGraphicsExtensions{.pdf,.jpeg,.png,.jpg}

\makeatletter
\DeclareRobustCommand\onedot{\futurelet\@let@token\@onedot}
\def\@onedot{\ifx\@let@token.\else.\null\fi\xspace}

\def\eg{\emph{e.g}\onedot} 
\def\ie{\emph{i.e}\onedot}

\def\etal{{et al}\onedot}
\def\GRASS#1#2{\mathcal{G}({#1},{#2})}

\makeatother

\def\Vec#1{{\boldsymbol{#1}}}
\def\Mat#1{{\boldsymbol{#1}}}

\def\citet{\cite}

\journalname{Int. Journal of Computer Vision}
\date{May  2015}
\titlerunning{Extrinsic Methods for Coding and Dictionary Learning on
Grassmann Manifolds}
\authorrunning{M.~Harandi \etal}

\newcommand{\tr}{\mathop{\rm  Tr}\nolimits}
\newcommand{\m}{\Mat}
\newcommand{\mh}[1]{\widehat{\m #1}}
\newcommand{\lip}{\left<\!\right.}
\newcommand{\rip}{\left.\!\right>}
\newcommand{\sk}[1]{[\Vec{#1}]}
\newcommand{\RIH}[2]{#2}		
\newcommand{\PGRASS}[2]{\mathcal{PG}({#1},{#2})}
\newcommand{\diag}{{\rm diag}}

\newcommand{\R}{\mathbf{R}}

\newcommand{\calG}{{\mathcal{G}}}
\newcommand{\calM}{{\mathcal{M}}}
\newcommand{\chord}{{\rm chord}}
\newcommand{\geo}{{\rm geod}}

\begin{document} 

\title{
      Extrinsic Methods for Coding and Dictionary Learning on Grassmann
Manifolds
      }

\author{
		Mehrtash~Harandi
			\and
		Richard~Hartley
			\and	
        Chunhua~Shen 
            \and 
        Brian~Lovell 
            \and 
 		Conrad~Sanderson
        }
        
\institute{M.~Harandi 
			and
		R.~Hartley \at College of Engineering and Computer Science, Australian
National University, and \\
    NICTA, Australia
    \\
    \email{\url{mehrtash.harandi@nicta.com.au}}\\
    \email{\url{richard.hartley@nicta.com.au}}
    \and
	C.~Shen  
	    \at	
    School of Computer Science, The University of Adelaide, SA 5005, Australia
    \\
    \email{\url{chunhua.shen@adelaide.edu.au}}    
    \and
	B.~Lovell
		    \at
    The University of Queensland,  Brisbane, Australia
    \\ \email{\url{lovell@itee.edu.au}}
    \and
    C. Sanderson
   	\at
	NICTA, Australia, and University of Queensland, Brisbane, Australia\\
   {\url{http://conradsanderson.id.au}}
}

\thanks{This work was in part supported by Australian Research Council
grant DP130104567.}

\maketitle

\begin{abstract}

Sparsity-based representations have recently led to notable results in
various visual recognition tasks.
In a separate line of research, Riemannian manifolds have been shown useful
for dealing with 
features and models that do not lie in Euclidean spaces.
With the aim of building a bridge between the two realms,
we address the problem of sparse coding and dictionary learning in Grassmann manifolds, \ie, the
space of linear subspaces.
To this end, we propose to embed Grassmann manifolds into the space of
symmetric matrices by an isometric mapping.
This in turn enables us to extend two sparse coding schemes to Grassmann
manifolds.
Furthermore, we {propose an algorithm for learning} a
Grassmann dictionary, atom by atom. 
Lastly, to handle non-linearity in data,
we extend the proposed Grassmann sparse coding and dictionary learning
algorithms through embedding into 
higher dimensional Hilbert spaces.

Experiments on several classification tasks
(gender recognition, gesture classification, scene analysis, face
recognition, action recognition and dynamic texture classification)
show that the proposed approaches achieve considerable improvements in
discrimination accuracy,
in comparison to state-of-the-art methods such as
kernelized Affine Hull Method and graph-embedding Grassmann discriminant
analysis.


\end{abstract}

\section{Introduction}
\label{sec:intro}

In the past decade, sparsity has become a popular term in neuroscience,
information theory, signal processing
and related areas~\cite{Olshausen_1996_Nature,CANDES_2006,Donoho_2006%
,Wright:PAMI:2009,ELAD_SR_BOOK_2010}.
Through sparse representation and compressive sensing it is possible to
represent natural signals
like images using only a few non-zero coefficients of a suitable basis.
In computer vision, sparse and overcomplete image representations were
first introduced for modeling the spatial receptive 
fields of simple cells in the human visual system
by~\citet{Olshausen_1996_Nature}.
The linear decomposition of a signal using a few atoms of a dictionary
has been shown to deliver notable results for various visual inference tasks,
such as face recognition~\cite{Wright:PAMI:2009,Wright_2010_IEEE}, image
classification~\cite{Yang_CVPR_2009,Mairal_2008_CVPR},
subspace clustering~\cite{Elhamifar_2013_PAMI}, image
restoration~\cite{Mairal_TIP_2008},
and motion segmentation~\cite{Rao_CVPR_2008} to name a few.
While significant steps have been taken to develop the theory of the sparse
coding and dictionary learning in Euclidean spaces,
similar problems on non-Euclidean geometry have received comparatively little 
attention~\cite{Yuan_ACCV_2009,Harandi_TNNLS15,Harandi_2013_ICCV%
,Guo_TIP_2013,Vemuri_ICML_2013,Cetingul_TMI_2014}.

This paper introduces techniques to sparsely represent
\mbox{$p$-dimensional} linear subspaces
in $\R^{d}$ using a combination of linear subspaces.
Linear subspaces can be considered as the core of
many inference algorithms in computer vision and machine learning.
For example, 
the set of all reflectance functions produced by Lambertian objects lies in
a linear subspace~\cite{Basri_2003_PAMI,Ramamoorthi_2002_PAMI}.
Several state-of-the-art methods for matching videos or image sets
model given data by subspaces~\cite{HAMM2008_ICML,Harandi_CVPR_2011%
,Turaga_PAMI_2011,Vemula:CVPR:2013,Sanderson_AVSS_2012,Shaokang:CVPR:2013}.
Auto regressive and moving average models, which are typically
employed to model dynamics in spatio-temporal processing,
can also be expressed by linear subspaces~\cite{Turaga_PAMI_2011}.
More applications of linear subspaces in computer vision include, 
chromatic noise filtering~\cite{Subbarao:IJCV:2009},
subspace clustering~\cite{Elhamifar_2013_PAMI},
motion segmentation~\cite{Rao_CVPR_2008},
domain adaptation~\cite{GFK:CVPR:2012,Gopalan:PAMI:2013},
and object tracking~\cite{Shirazi_2015}.

Despite their wide applications and appealing properties, subspaces lie on
a special type of Riemannian manifold, 
namely the Grassmann manifold, which makes their analysis very challenging.
This paper tackles and provides efficient solutions to the following
two fundamental problems on Grassmann manifolds (see
Fig.~\ref{fig:conceptual} for a conceptual illustration): 
\begin{enumerate} \itemsep2pt
\item 
\textbf{Coding.} 
Given a subspace $\mathcal{X}$ and a set $\mathbb{D} =
\{\mathcal{D}_i\}_{i=1}^N$ with $N$ elements (also known as atoms),  
where $\mathcal{X}$ and $\mathcal{D}_i$ are linear subspaces,
how can $\mathcal{X}$ be approximated by a combination of atoms
in~$\mathbb{D}$~?

\item 
\textbf{Dictionary learning.} 
Given a set of subspaces $\{\mathcal{X}_i\}_{i=1}^m$,
how can a smaller set of subspaces $\mathbb{D} = \{\mathcal{D}_i\}_{i=1}^N$
be learned to represent $\{\mathcal{X}_i\}_{i=1}^m$ accurately?  

\end{enumerate}

Our main motivation here is to develop new methods for analyzing video
data and image sets. 
This is inspired by the success of sparse signal modeling and related
topics that suggest natural signals like images (and hence video and
image sets as our concern here) can be efficiently
approximated by superposition of atoms of a dictionary.
We generalize the traditional notion of coding, which operates on vectors,
to coding on subspaces. 
Coding with the dictionary of subspaces can then be seamlessly  
used for categorizing video data.

Considering the problem of coding and dictionary learning on Grassmann manifolds,
previous studies (\eg,~\cite{Vemuri_ICML_2013,Cetingul_ISBI_2011,Cetingul_TMI_2014}) 
opt for an intrinsic and general framework for sparse coding on Riemannian manifolds.
This intrinsic formulation exploits the tangent bundle of the manifold for
sparse coding.
Due to the computational complexity of the logarithm map on Grassmann manifolds, 
performing intrinsic sparse coding might be computationally demanding for
the problems that we are interested in (\eg, video analysis). 
Moreover, learning a dictionary based on the intrinsic formulation as
proposed by~\citet{Vemuri_ICML_2013} 
requires computing the gradient of a cost function that includes terms
based on logarithm map. 
As will be  shown later, the involvement of logarithm map 
(which does not have an analytic formulation on Grassmann manifolds)
deprives us from having a closed-form solution 
for learning a Grassmann dictionary intrinsically.

{\bf Contributions}.
In light of the above discussion, in this paper we introduce an extrinsic
methods for coding and dictionary learning on Grassmann manifolds.
To this end, we propose to embed Grassmann manifolds into the space of
symmetric matrices
by a diffeomorphism that preserves several properties of the Grassmannian structure.
We show how coding can be accomplished in the induced space
and devise an algorithm for updating a Grassmann dictionary atom by atom.
Furthermore, in order to accommodate non-linearity in data, we propose
kernelized versions of our coding and dictionary learning algorithms.
Our contributions are therefore three-fold: 
\begin{itemize}
\item[1.]
    We propose to perform coding and dictionary learning for data points on
    Grassmann manifolds by embedding the manifolds into the space of
    symmetric matrices.
\item[2.]
    We derive kernelized versions of the proposed coding and dictionary
learning algorithms (\ie, embedded into Hilbert spaces), 
    which can address non-linearity in data.
\item[3.]
    We apply the proposed Grassmann dictionary learning methods to
    several computer vision tasks where the data are videos or image sets.  
    Our proposed algorithms outperform  state-of-the-art methods
    on a wide range of classification tasks,
    including gender recognition from gait, scene analysis, 
    face recognition from image sets, action recognition and dynamic
texture classification.
\end{itemize}

\begin{figure}[!tb]
  \centering
  \includegraphics[width=1\columnwidth,keepaspectratio]
{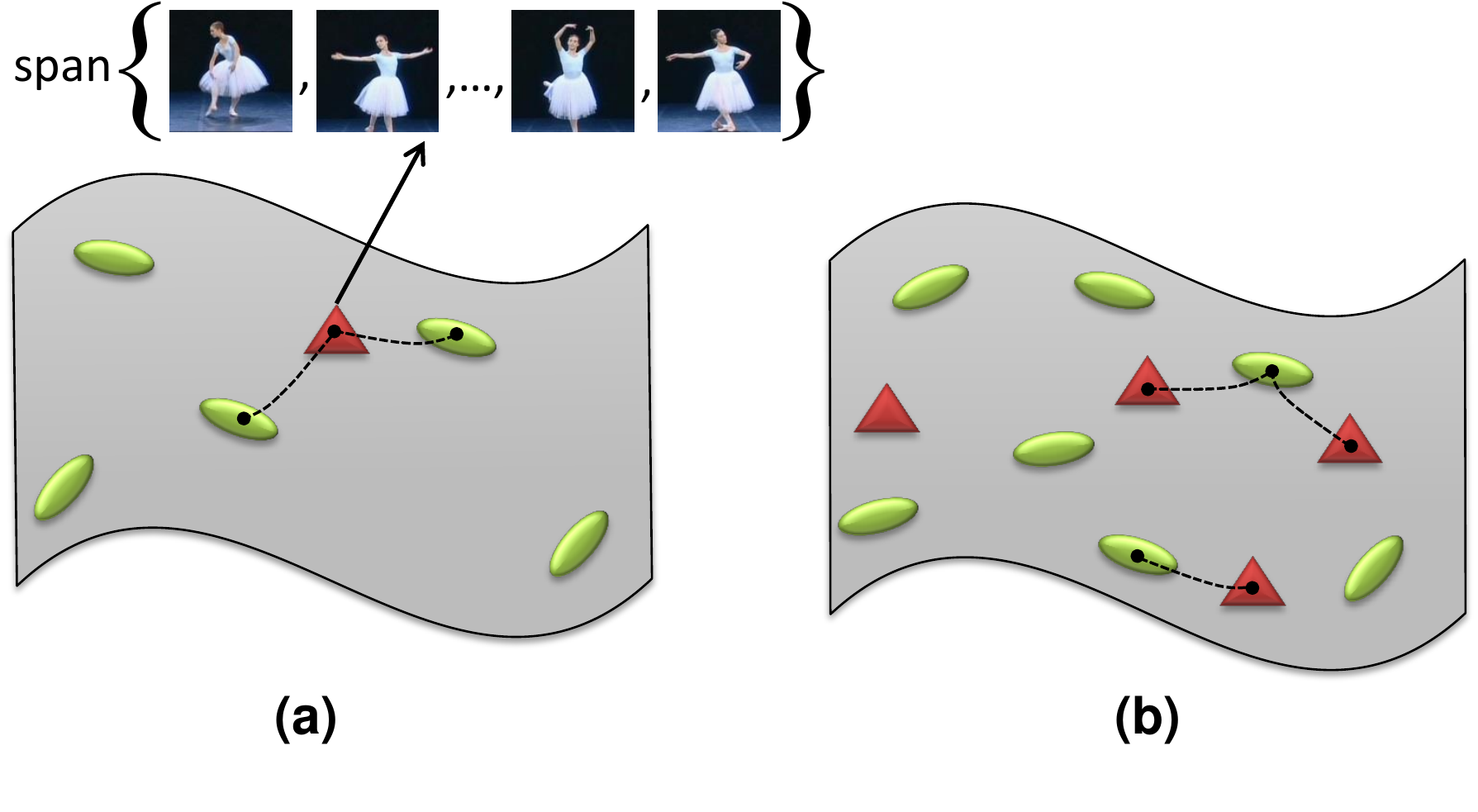}   
  \caption
    { \small A conceptual diagram of the problems addressed in this work.
    A~video or an image set can be modeled by a linear subspace,
    which can be represented as a point on a Grassmann manifold.
    {\bf (a)~Sparse coding on a Grassmann manifold.}
    Given a dictionary (green ellipses) and a query signal (red triangle) 
    on the Grassmann manifold,
    we are interested in estimating the query signal by a sparse
combination of atoms
    while taking into account the geometry of the manifold (\eg, curvature).
    {\bf (b)~Dictionary learning on a Grassmann manifold.}
    Given a set of observations (green ellipses) on a Grassmann manifold, 
    we are interested in determining a dictionary (red triangles) to describe 
    the observations sparsely, while taking into account the geometry.
    This figure is best seen in color.
    }
  \label{fig:conceptual}
\end{figure}

\section{Background Theory}
\label{sec:background}

This section overviews Grassmann geometry and provides the groundwork for
techniques described in following sections.
Since the term ``manifold'' itself is often used in computer vision in a
somewhat loose sense, we emphasize 
that the word is used in this paper in its strict mathematical sense.

Throughout the paper, bold capital letters denote matrices (\eg $\Mat{X}$)
and bold lower-case letters denote column vectors (\eg $\Vec{x}$).
The notations $[ \cdot ]_i$ and $[\cdot]_{i,j}$ are used to demonstrate
elements in position $i$ and $(i,j)$ in a vector and matrix, respectively.
$\Vec{1}_{d} \in \R^{d}$ and $\Vec{0}_{d} \in \R^{d}$ are
vectors of ones and zeros.
$\mathbf{I}_d$ is the $d \times d$ identity matrix.
$\Vert \Vec{x} \Vert_1 = \sum\nolimits_i|[x]_i|$ and $\Vert \Vec{x} \Vert =
\sqrt{\Vec{x}^T\Vec{x}}$
denote the $\ell_1$ and $\ell_2$ norms, respectively, with $T$ indicating
transposition.
$\Vert \Mat{X} \Vert_F = \sqrt{\tr \big(\Mat{X}^T\Mat{X}\big)}$ designates
the Frobenius norm,
with $\tr(\cdot)$ computing the matrix trace.


\subsection{Grassmann Manifolds and their Riemannian Structure}

For $0 < p \leq d$, the space of $d \times p$ matrices with orthonormal columns is not a Euclidean space
but a Riemannian manifold, the Stiefel manifold $\mathrm{St}(p,d)$.
That is,
\begin{equation}
    \mathrm{St}(p,d) \triangleq \{ \Mat{X} \in \R^{d \times p} :
\Mat{X}^T \Mat{X} = \mathbf{I}_p\}.
    \label{eqn:def_stiefel}
\end{equation}
By grouping together all points on $\mathrm{St}(p,d)$ that span the same
subspace
we obtain the Grassmann manifold~$\GRASS{p}{d}$. More formally,
the Stiefel manifold $\mathrm{St}(p,d)$ admits a right action by
the orthogonal group $O(p)$ (consisting of $p\times p$ orthogonal matrices);
for $\m X \in \mathrm{St}(p,d)$ and $\m U\in O(p)$, the matrix
$\m X \m U$ is also an element of $\mathrm{St}(p,d)$.  Furthermore
the columns of $\m X$ and $\m X \m U$ span the same subspace of $\R^d$,
and are to be thought of representatives of the same element of the
Grassmann manifold,
$\GRASS{p}{d}$. Thus, the orbits of this group action form the elements
of the Grassman manifold.
The resulting set of orbits is a manifold according
to the quotient manifold theorem (see Theorem 21.10 in~\cite{Lee_Book_2012}).
The details of this construction are not critical to an understanding
of the rest of this paper.


%

%
An element $\mathcal{X}$ of $\GRASS{p}{d}$ can be specified by a
basis, \ie, a set of $p$ vectors $\Vec{x}_1,~\cdots,~\Vec{x}_p$ such that
$\mathcal{X}$ is the set of all their linear combinations.
When the $\Vec{x}$ vectors are ordered as the columns of a $d \times p$ matrix
$\Mat{X}$, then $\Mat{X}$ is said to span
$\mathcal{X}$ and  we write $\mathcal{X} = \mathrm{span}(\Mat{X})$. In what
follows, we refer to a subspace $\mathcal{X}$ and hence 
a point on $\GRASS{p}{d}$ by its $d \times p$ basis matrix $\Mat{X}$.
The choice of the basis is 
not unique but it has no effect in what we develop later.

A Riemannian metric on a manifold is defined formally as
a smooth inner product on the tangent bundle.
(See~\citet{Absil_2004} for the form of Riemannian metric on $\GRASS{p}{d}$).
However, we shall be
concerned only with geodesic distances on
the Grassmann manifold, which allows us to avoid many technical points and give a straight-forward definition.

%

%

On a Riemannian manifold, points are connected via smooth curves. The
geodesic distance between two points is
defined as the length of shortest curve in the manifold (called a {\em geodesic}) connecting them.
The Stiefel manifold $\mathrm{St}(p,d)$ is embedded in the
set of $d\times p$ matrices, which may be seen as a 
Euclidean space $\R^{d \times p}$ with
distances defined by the Frobenius norm.  Consequently the length of a smooth
curve (or path) in $\mathrm{St}(p,d)$ is defined as its length as a curve in 
$\R^{d \times p}$. Now, given two points $\mathcal X$ and $\mathcal Y$ in
$\GRASS{p}{d}$, the distance $d_\geo(\mathcal X, \mathcal Y)$ is defined as
the length of the shortest path in $\mathrm{St}(p,d)$ between
any two points $\m X$ and $\m Y$ in $\mathrm{St}(p,d)$
that are members of the equivalence clases $\mathcal X$ and
$\mathcal Y$.


The geodesic distance has an interpretation as the magnitude of the smallest
rotation 
that takes one subspace to the other. If {$\Theta = [\theta_1, \theta_2,
\ldots, \theta_p]$}
is the sequence of principal angles between two subspaces $\mathcal{X}_1
\in \GRASS{p}{d}$ and 
$\mathcal{X}_2 \in \GRASS{p}{d}$, then
 $d_\geo\left(\mathcal{X}_1,\mathcal{X}_2\right)=\|\Theta\|_2$.


\begin{definition}[Principal Angles]
Let $\Mat{X}_1$ and $\Mat{X}_2$ be two matrices of size $d \times p$ with
orthonormal columns. The principal angles
$0 \leq \theta_1 $ $\leq \theta_2 \leq $ $\cdots $ $ \leq \theta_p \leq
\pi/2$ between two subspaces
$\operatorname{Span}(\Mat{X}_1)$ and $\operatorname{Span}(\Mat{X}_2)$, are
defined recursively by
\end{definition}
\begin{eqnarray}
  &\cos(\theta_i)
  =
  \underset{\Vec{u}_i \in \operatorname{Span}(\Mat{X}_1)}{\max}\;
  \underset{\Vec{v}_i \in \operatorname{Span}(\Mat{X}_2)}{\max}\;
  \Vec{u}_i^T \Vec{v}_i  \\
  \text{s.t.:}
  &\|\Vec{u}_i\|_2 \mbox{~=~} \|\Vec{v}_i\|_2 \mbox{~=~} 1  \nonumber\\
  &\Vec{u}_i^T \Vec{u}_j \mbox{~=~} 0;\; j=1,2,\cdots,i-1     \nonumber\\
  &\Vec{v}_i^T \Vec{v}_j \mbox{~=~} 0;\; j=1,2,\cdots,i-1     \nonumber
  \label{eqn:Principal_Angle}
\end{eqnarray}%
In other words, the first principal angle $\theta_1$ is the smallest
angle between all pairs of unit vectors in the first and the second
subspaces. The rest of the principal
angles are defined similarly.

Two operators, namely the logarithm map $\log_{\Vec{x}}(\cdot):~\calM
\to T_{\Vec{x}}(\calM)$
and its inverse, the exponential map $\exp_{\Vec{x}}(\cdot):T_{\Vec{x}}(
\calM) \to \calM$
are defined over {Riemannian manifolds} to switch between
the manifold and the tangent space at $\Vec{x}$.
A key point here is the fact that both the logarithm map and its inverse do
not have closed-form solutions for Grassmann manifolds.
Efficient numerical approaches for computing both maps were proposed
by~\cite{ANUJ_2003,Begelfor_CVPR_2006}. In this paper, however, 
the exponential and logarithm maps will only be used when describing
previous work of other authors.


\section{Problem Statement}
\label{sec:problem_def}

In vector spaces, by {\it coding} we mean the general notion of
representing a vector $\Vec x$ (the {\em query}) as some combination of
other vectors $\Vec d_i$ belonging to a {\em dictionary}. Typically,
$\Vec x$ is expressed as a linear combination
$\Vec x = \sum_{j=1}^N \sk{y}_j \Vec d_j$,
or else as an {\em affine combination} in which the coefficients
$\sk{y}_j$ satisfy the additional constraint $\sum_{j=1}^N \sk{y}_j = 1$.
(This constraint may also be written as $\Vec 1^T \Vec y = 1$.)

In {\it sparse coding} one seeks to express the query in terms of a small
number of dictionary elements.
Given a query $\Vec{x} \in \R^d$ and a dictionary $\mathbb{D}$ of
size $N$, \ie, 
$\mathbb{D}_{d \times N} = \{ \Vec{d}_1, \Vec{d}_2, \cdots, \Vec{d}_N
\}$
with atoms $\Vec{d}_i \in \R^d$, the problem of coding $\Vec x$
can be formulated as solving the minimization problem:


\begin{equation}
	l_E(\Vec{x},\mathbb{D}) \triangleq 
    \underset{\Vec{y}}{\min} \:
    \Bigl\| \Vec{x}- \sum\nolimits_{j=1}^{N} \sk y_{j} \Vec{d}_j
    \Bigr\|_2^2
    + \lambda f( \Vec{y} ). 
    \label{eqn:euc_coding}
\end{equation}
The domain of $\Vec y$ may be the whole of $\R^N$, so that the
sum runs over all linear combinations of dictionary elements (or {\em atoms}), or
alternatively, the extra constraint $\Vec 1^T \Vec y$ may be specified,
to restrict to affine combinations.

The idea here is to (approximately) reconstruct the query $\Vec{x}$ by a
combination of dictionary atoms 
while forcing the coefficients of combination, \ie, $\Vec{y}$, to have
some structure.  The quantity $l_E(\Vec{x},\mathbb{D})$ can be thought 
of as a coding cost combining the squared residual coding error,
reflected in the energy term $\|\cdot\|_2^2$ in~\eqref{eqn:euc_coding}, along with a penalty term
$f( \Vec{y})$, which encourages some structure such as sparsity.
The function $f:\R^N \to \R$
could be the $\ell_1$ norm, as in the Lasso
problem~\cite{Tibshirani_1996},
or some form of locality as proposed by~\citet{Yu:NIPS:2009}
and~\citet{Wang_CVPR_2010_LLC}.

The problem of dictionary learning is to determine $\mathbb{D}$
given a finite set of observations 
$\left \{ \Vec{x}_i  \right \}_{i=1}^{m}, \: \Vec{x} \in
\R^d$,
by minimizing the total coding cost for all observations, namely
\begin{equation}
	h(\mathbb{D}) \triangleq \sum\limits_{i=1}^{m} l_E(\Vec{x}_i,\mathbb{D})
    \label{eqn:euc_dic_learning} ~.
\end{equation}
A ``good'' dictionary has a small residual coding error for all
observations $\Vec x_i$ while 
producing codes $\Vec{y}_i \in \R^N$ with the desired
structure. For example, 
in the case of sparse coding, the $\ell_1$ norm is usually taken as
$f(\cdot)$
to obtain the most common form of dictionary learning in the literature.
More specifically, the sparse dictionary learning problem may be written in full as that of
jointly minimizing the total coding cost over all choices of coefficients
and dictionary:
\begin{equation}
\label{eqn:euc_dic_learning2}
\underset{ \{\Vec{y}_i \}_{i=1}^m, \mathbb{D}}{\min} \:
    \sum\limits_{i=1}^{m}\,\,\Bigl\| \Vec{x}_i- \sum\limits_{j=1}^{N}
\sk{y_i}_j \Vec{d}_j \Bigr\|_2^2
    +\lambda 
    \sum\limits_{i=1}^{m} \|\Vec{y}_i\|_1.  
\end{equation}
A common approach to solving this is to alternate
between the two sets of variables, $\mathbb{D}$ and $\{\Vec{y}_i \}_{i=1}^m$, 
as proposed for example by~\cite{Aharon_2006_ksvd}
(see~\cite{ELAD_SR_BOOK_2010} for a detailed treatment).
Minimizing~\eqref{eqn:euc_dic_learning2} over sparse
codes $\Vec{y}_i$
while dictionary $\mathbb{D}$ is fixed is a convex problem. Similarly,
minimizing the overall problem
over $\mathbb{D}$ with fixed $\{\Vec{y}_i \}_{i=1}^m$ is convex as well.

In generalizing the coding problem to a more general space $\mathcal{M}$,
(\eg, Riemannian manifolds), one may write~\eqref{eqn:euc_coding} as 
\begin{equation}
\label{eqn:grass_coding}
	l_\calM(\mathcal{X},\mathbb{D}) \triangleq  
	\underset{\Vec{y}}{\min} \:  \Big(
    d_\calM
\big( \mathcal{X}, \,
    C(\Vec y, \mathbb{D}) \big)^2
    +\lambda f(\Vec{y})\Big) .      
\end{equation}

Here $\mathcal{X}$ and $\mathbb{D} =
\{\mathcal{D}_j\}_{j=1}^N$
are points in the space $\calM$, while $d_\calM(\cdot, \cdot)$
is some distance metric and 
$C: \R^N \times \calM^N \to \calM$ is an {\em encoding function}, 
assigning an element of $\calM$ to every choice of coefficients
and dictionary.  Note that~\eqref{eqn:euc_coding} is a special case
of this, in which $C(\Vec y, \mathbb{D})$ represents linear or affine 
combination,
and $d_\calM(\cdot, \cdot)$ is the Euclidean distance metric.
To define the coding, one need only
specify the metric $d_\calM(\cdot, \cdot)$ to be used and the 
encoding function $C(\cdot, \cdot)$.  Although this formulation may apply to
a wide range of spaces, here we shall be concerned chiefly with
coding on Grassmann manifolds.


\section{Related Work}
\label{subsec:related_work}

A seemingly straightforward method for coding and dictionary learning
is through embedding manifolds into Euclidean spaces via a fixed tangent space.
The embedding function in this case would be {$\log_{\mathcal{P}}(\cdot)$},
where {$\mathcal{P}$} is some default base point.
The natural choice for the base point on $\GRASS{p}{d}$ is
\begin{equation*}
\mathcal{P}= \mathrm{span} \Bigg(
\begin{bmatrix}
\mathbf{I}_{p \times p}\\
\Mat{0}_{(d-p) \times p}
\end{bmatrix}
\Bigg)\;.
\end{equation*}

By mapping points in the manifold $\calM$ to the tangent space,
the problem at hand is transformed to its Euclidean counterpart.
For example in the case of sparse coding, instead of~\eqref{eqn:grass_coding}, 
the encoding cost may be defined as follows:
\begin{align}
\begin{split}
\label{eqn:log_euc_grass_coding}
	l_{\calM}(\mathcal{X},\mathbb{D}) \triangleq  
	\underset{\Vec{y}}{\min} \:  
    \Bigl\| \log_\mathcal{P}(\mathcal{X}) -
    \sum\limits_{j=1}^{N} \sk y_{j} \log_\mathcal{P}(\mathcal{D}_j)
\Bigr\|^2_\mathcal{P}
    +\lambda f(\Vec{y})   
\end{split}
\end{align}
where the notation $\|\cdot\|_{\mathcal{P}}$ reminds us that
the norm is in the tangent space at $\mathcal X$.
We shall refer to this straightforward approach as {\it Log-Euclidean}
sparse coding 
(the corresponding steps for Grassmann manifolds in Algorithm~\ref{alg:alg_log_euc_sc}),
following the terminology used in~\cite{Arsigny:2006}.
This idea has been deployed for action recognition on the manifold of
Symmetric Positive Definite matrices
by~\citet{Yuan_ACCV_2009} and~\citet{Guo_TIP_2013}.
Since on a tangent space only distances to the base point are
equal to true geodesic distances, the
Log-Euclidean solution does not take into account the true structure of the
underlying Riemannian manifold.  Moreover, the solution is
dependent upon the particular point $\mathcal P$ used as a base point.

\begin{algorithm}[!tb]\small
	\SetKwBlock{KwInit}{Initialization.}{}
	\SetKwBlock{KwProc}{Processing.}{}
	\SetAlgoLined	
	\setcounter{AlgoLine}{1}	
	\KwIn
	{Grassmann dictionary $\{ \mathcal{D}_i
\}_{i=1}^N,~\mathcal{D}_i~\in~\GRASS{p}{d}$; 
	the query sample $\mathcal{X}\in~\GRASS{p}{d}$.
	}
	\KwOut
	{The sparse code {$\Vec{y}^\ast$}.
	}
	\BlankLine
	\KwInit{
	\For{$i \gets 1$ \KwTo $N$}{
	\mbox{$\Vec{d}_i \gets \log_{\mathcal{P}}(\mathcal{D}_i)$}\;
	}
	\mbox{$\Mat{A} \gets \left[\Vec{d}_1 | \Vec{d}_2 | \cdots |\Vec{d}_N
\right]$} \;
	}
	\KwProc{
	\mbox{$\Vec{x} \gets \log_{\mathcal{P}}(\mathcal{X})$}\;
	\mbox{$\Vec{y}^\ast \gets \arg\min_{\Vec{y}} 
	\left\| \Mat{x}-\Mat{A}^Ty \right\|_2^2 + \lambda \left\| \Vec{y}
\right\|_1$}\;	
	}
	\caption{\small Log-Euclidean sparse coding on Grassmann manifolds.}
	\label{alg:alg_log_euc_sc}
\end{algorithm} 

A more elegant and intrinsic approach 
is to work in
the tangent bundle of the manifold, varying the particular tangent space
according to the point $\mathcal{X}$ being approximated.  Such an idea 
has roots in the work of~\citet{Goh_CVPR_2008} which extends various
methods of dimensionality reduction 
to Riemannian manifolds.
As for sparse coding,~\citet{Cetingul_ISBI_2011},~\citet{Cetingul_TMI_2014}
and \citet{Vemuri_ICML_2013} show that 
by working in the tangent space at $\mathcal{X}$, \ie,
$T_\mathcal{X}(\calM)$, the encoding cost in~\eqref{eqn:grass_coding} can be written as 
\begin{align}
\label{eqn:intrinsic_sc}
l_\calM(\mathcal{X},\mathbb{D}) \triangleq &
\min_{\substack{\Vec y\in\R^N \\ \Vec{1}^T_{}\Vec{y} = 1}}
\,\,\Big\|\sum_{j=1}^{N} \sk y_j
\log_{\mathcal{X}}(\mathcal{D}_j) \Big\|^2_{\mathcal{X}}  + \lambda f(\Vec{y})	
\end{align}
To see the relationship between~\eqref{eqn:grass_coding} and~\eqref{eqn:intrinsic_sc}, note that
$\log_{\mathcal{X}}(\mathcal{D}_j)$ is unambiguously defined for
most pairs $(\mathcal{X},\mathcal{D}_j)$.  
If the encoding function $C(\cdot, \cdot)$ is defined by
\[
C(\Vec y, \mathbb{D}) = \exp_{\mathcal X} \Big(\sum_{j=1}^{N} \sk y_j
\log_{\mathcal{X}}(\mathcal{D}_j)\Big)\]
then
\[
\Big\|\sum_{j=1}^{N} \sk y_j
\log_{\mathcal{X}}(\mathcal{D}_j) \Big\|_\mathcal{X} = \| \log_{\mathcal X}\,C(\Vec y, \mathbb{D}) \|_{\mathcal X}
= d_\geo (\mathcal{X}, C(\Vec y, \mathbb{D}) ) ~.
\]
The extra affine constraint, \ie, $\Vec{1}^T_{}\Vec{y} = 1$ is necessary
to avoid a trivial solution and has been used successfully in other applications such as
dimensionality reduction~\cite{Roweis_2000_LLE}, 
subspace clustering~\cite{Elhamifar_2013_PAMI} and
coding~\cite{Yu_2010_ICML,Wang_CVPR_2010_LLC} 
to name a few.

Turning our attention to the problem of dictionary learning, given a set of
training data 
$\{\mathcal{X}_i\}_{i=1}^m,\; \mathcal{X}_i \in \calM$,
recasting the problem of~\eqref{eqn:euc_dic_learning2} to the
Riemannian manifold $\calM$ by following ~\cite{Vemuri_ICML_2013} results in
\begin{align}
	&\underset{ \{\Vec{y}_i \}_{i=1}^m, \mathbb{D}}{\min} \:
    \sum\limits_{i=1}^{m}\bigg\| \sum\limits_{j=1}^{N}\sk{y_i}_j
\log_{\mathcal{X}_i}(\mathcal{D}_j)\bigg\|^2
    +\lambda     \sum\limits_{i=1}^{m} \|\Vec{y}_i\|_1 
    \label{eqn:dic_learn_intrinsic_riemann}\\	
	&\mathrm{s.t.}~~~\Vec{1}^T_{}\Vec{y}_i = 1,\; i = 1,2,\cdots,m.\nonumber
\end{align}
Similar to the Euclidean case, the problem in~\eqref{eqn:dic_learn_intrinsic_riemann} is solved by iterative
optimization over $\{\Vec{y}_i\}_{i= 1 }^m$ and $\mathbb{D}$. Computing the
sparse codes $\{\Vec{y}_i\}_{i=1}^m$ is done by solving~\eqref{eqn:intrinsic_sc}. To update $\mathbb{D}$,~
\citet{Vemuri_ICML_2013} proposed a gradient descent approach 
along geodesics. That is, the update of $\mathcal{D}_r$ at time
$t$ while $\{\Vec{y}_i\}_{i=1}^m$ and $\mathcal{D}_j, j \neq r$ are kept
fixed has the form 
\begin{equation}
	\mathcal{D}_r^{(t)} = \exp_{\mathcal{D}_r^{(t-1)}}(-\eta\Delta).
	\label{eqn:intrinsic_gradient_descent}
\end{equation}
In Eq.~\eqref{eqn:intrinsic_gradient_descent} $\eta$ is a step size and the
tangent vector $\Delta:\R \to T_{\mathcal{D}_r}(\calM)$
represents the direction of maximum ascent. That is $\Delta =
\mathrm{grad} \mathcal{J}(\mathcal{D}_r)$%
\footnote{On an abstract Riemannian manifold $\calM$, the gradient of
a smooth real function $f$ at a point
$x \in \calM$, denoted by $\mathrm{grad} f(x)$, 
is the element of $T_x(\calM)$ satisfying
$\langle \mathrm{grad}f(x), \zeta \rangle_x = Df_x[\zeta]$ for all $\zeta
\in T_x(\calM)$. Here,
$Df_x[\zeta]$ denotes the directional derivative of $f$ at $x$ in the
direction of $\zeta$. The interested reader is
referred to~\cite{Absil:2008} for more details on how the gradient of a
function on Grassmann manifolds can be computed.}%
, where 
\begin{equation}
\mathcal{J} = \sum\limits_{i=1}^{m}\Big\| \sum\limits_{j=1}^{N}\sk{y_i}_j
\log_{\mathcal{X}_i}(\mathcal{D}_j)\Big\|^2.
\end{equation}%

Here is where the difficulty arises. Since the logarithm map does not have
a closed-form expression on Grassmann manifolds, 
an analytic expression for $\Delta$ in Eq.~
\eqref{eqn:intrinsic_gradient_descent} cannot be sought for the case of
interest
in this work, \ie, Grassmann manifolds%
\footnote{This is acknowledged by~\citet{Vemuri_ICML_2013}.}.
Having this in mind, we propose extrinsic approaches to coding and
dictionary learning specialized for Grassmann manifolds.
Our proposal is different from the intrinsic method in following points:
\begin{itemize}
\item As compared to the intrinsic approach, our extrinsic coding methods
are noticeably faster. This is especially attractive
for vision applications where the dimensionality of Grassmann manifolds is
high.
\item 
{Similar to the intrinsic method, our proposed dictionary
learning approach is an alternating method.
However and in contrast to the intrinsic method, the updating rule for
dictionary atoms admits
an analytic form.}
\item Our proposed extrinsic methods can be kernelized. Such kernelization
for the intrinsic method is not possible due to the fact that the logarithm
map does not have a closed-form and analytic expression on Grassmann
manifolds. Kernelized coding enables us to model non-linearity in data
better (think of samples that do not lie on a subspace in low-dimensional
space but could form one in a higher-possibly infinite- dimensional space).
As shown in our experiments, kernelized coding can result in higher
recognition accuracies 
as compared to linear coding.
\end{itemize}

\section{Coding on Grassmann Manifolds}
\label{sec:Coding}

In this work, we propose to embed Grassmann manifolds into the space of
symmetric matrices  
via the \emph{projection embedding}~\cite{Chikuse:2003}.
The projection embedding has been previously used in 
subspace tracking~\cite{Srivastava:2004},
clustering~\cite{Cetingul:CVPR:2009},
discriminant analysis~\cite{HAMM2008_ICML,Harandi_CVPR_2011} 
and classification purposes~\cite{Vemula:CVPR:2013}.
Let $\PGRASS{p}{d}$ be the set
of $d\times d$ idempotent and symmetric matrices of rank $p$. 
The projection embedding $\Pi: \GRASS{p}{d} \rightarrow \PGRASS{p}{d}$ is given by 
$\Pi(\mathcal{X}) = \Mat{X}\Mat{X}^T$ where $\mathcal{X} = \mathrm{span}(\Mat{X})$.
The mapping $\Pi$ is a diffeomorphism~\cite{Chikuse:2003},
and $\PGRASS{p}{d}$ may be thought of as simply an alternative form of
the Grassmann manifold. It is a smooth, compact submanifold of
$\rm{Sym}(d)$ of dimension $d(d - p)$~\cite{Helmke:2007}.

From its embedding in $\mathrm{Sym}(d)$, the manifold $\PGRASS{p}{d}$
inherits a Riemannian metric (and hence a notion of path length),
from the Frobenius norm in $\mathrm{Sym}(d)$.
It is an important fact that the mapping $\Pi$ is also an isometry
with respect to the Riemannian metric on $\PGRASS{p}{d}$
and the standard Riemannian metric, defined in
Section~\ref{sec:background} for $\GRASS{p}{d}$~\cite{Chikuse:2003}.  Hence, $\Pi$ preserves length of curves~\cite{Harandi_2013_ICCV}. 
The shortest path length between two points in $\PGRASS{p}{d}$ defines
a distance metric called the geodesic metric.

Working with $\PGRASS{p}{d}$ instead of
$\GRASS{p}{d}$ has the advantage that each element of $\PGRASS{p}{d}$
is a single matrix, whereas elements of $\GRASS{p}{d}$ are 
equivalence classes of matrices. In other words,
if $\Mat{X}$ and $\Mat{X}^\star = \Mat{X}\Mat{R},~\Mat{R} \in \mathrm{O}(p)$ are two bases for $\mathcal{X}$, 
then $\Pi(\Mat{X}) = \Pi(\Mat{X}^\star)$.

In future, we shall denote $\m X \m X^T$ by $\mh X$, the
hat representing the action of the projection embedding. 
Furthermore,
$\lip\cdot, \cdot\rip$
represents the Frobenius inner product: thus
$\lip\mh X, \mh Y\rip = \tr (\mh X \mh Y)$. Note that in computing
$\lip\mh X, \mh Y\rip$ it is not necessary to compute $\mh X$
and $\mh Y$ explicitly (they may be large matrices).  Instead, note that
$\lip\mh X, \mh Y\rip = \tr(\mh X \mh Y) 
= \tr (\m X \m X^T \m Y \m Y^T)
= \tr(\m Y^T \m X \m X^T \m Y) = \| \m Y^T \m X \|_F^2$. This is advantageous,
since $\m Y^T \m X$ may be a substantially smaller matrix.

Apart from the geodesic distance metric, 
an important metric used in this paper is
the {\em chordal metric}, defined by 
\begin{align}
    d_\chord(\mh X, \mh Y) &= \|\Pi(\mathcal{X}) -
\Pi(\mathcal{Y})\|_F 
    = \|\mh X - \mh Y\|_F\;,
    \label{eqn:proj_dist}
\end{align}
This metric will be used in the context of~\eqref{eqn:grass_coding}
to recast the coding and consequently dictionary-learning problem in terms of chordal distance.
Before presenting our proposed methods, 
we establish an interesting link between coding and the notion of weighted mean in a metric space.

\paragraph{\bf Weighted Karcher mean. }
The underlying concept of coding using a dictionary is to represent
in some way a point in a space of interest as a combination of other elements
in that space.  In the usual method of coding in $\R^d$ given by
(\ref{eqn:euc_coding}), each $\Vec{x}$ is represented by a linear combination
of dictionary elements $\Vec{d}_j$, where the first term represents the coding
error. For coding in a manifold, the problem to address is that linear 
combinations do not make sense.  We wish to find some way in which an element
$\mathcal{X}$ may be represented in terms of other dictionary elements $\mathcal{D}_j$
as suggested in (\ref{eqn:grass_coding}). For a
proposed method to generalize the $\R^d$ case, one may prefer a method that
is a direct generalization of the Euclidean case in some way.

In $\R^d$, a different way to consider the expression  
$\sum_{j=1}^N \sk y_{j} \Vec{d}_j$ in~\eqref{eqn:euc_coding} is as a weighted mean of the points $\Vec{d}_j$
This observation relies on the following fact, which is verified using
a Lagrange multiplier method.
\begin{lemma}
Given coefficients $\Vec y$ with $\sum_{i=1}^N \sk{y}_i = 1$, 
and dictionary elements $\{ \Vec d_1, \ldots \Vec d_N\}$ in $\R^d$, the point $\Vec x^*\in\R^d$ that minimizes
$\sum_{i=1}^N \, \sk{y}_i \, \| \Vec x - \Vec d_i\|_F^2$
is given by $\Vec x^* = \sum_{i=1}^N \sk{y}_i \, \Vec d_i$.
\end{lemma}
In other words, the affine combination of dictionary elements is
equal to their weighted mean.  Although linear combinations
are not defined for points on manifolds or metric spaces, a weighted mean is.
\begin{definition}
\label{def:Karcher-mean}
Given points $\mathcal{D}_i$ on a Riemannian manifold $\mathcal{M}$, 
and weights $\sk y_i$, the point $\cal X^*$ that minimizes
$\sum_{i=1}^N \sk y_i \, d_g (\mathcal{X}, \mathcal{D}_i)^2$, 
is called the weighted Karcher mean of the points $\mathcal{D}_i$ with
weights $\sk y_i$. Here, $d_g (\cdot, \cdot)$ is the geodesic distance on $\mathcal{M}$.
\end{definition}

Generally, finding the Karcher mean~\cite{Karcher_1977} on a manifold 
involves an iterative
procedure, which may converge to a local minimum, even on a simple
manifold, such as $SO(3)$ \cite{Manton_2004,Hartley_IJCV_13}. 
However, one may replace the geodesic metric with a different metric 
in order to simplify the calculation. To this end, we propose the
{\em chordal metric} on a Grassman manifold, defined for matrices
$\mh X$ and $\mh Y$ in $\PGRASS{p}{n}$ by
\begin{equation}
\label{eq:chordal-metric}
d_\chord (\mh X, \mh Y) = \| \mh X - \mh Y\|_F ~.
\end{equation} 
The corresponding mean, as in definition~\ref{def:Karcher-mean} (but 
using the chordal metric) is called the {\em weighted chordal mean}
of the points.  In contrast to the Karcher mean, the weighted
chordal mean on a Grassman manifold has a simple closed-form.
\begin{theorem}
\label{thm:chordal-mean}
The  weighted chordal mean of a set of points $\mh D_i \in\PGRASS{p}{d}$ with weights $\sk y_i$
is equal to 
\(
{\rm Proj} (\sum_{i=1}^m \sk y_i \mh D_i),
\)
where ${\rm Proj}(\cdot)$ represents the closest point on
$\PGRASS{p}{d}$.
\end{theorem}
The function ${\rm Proj}(\cdot)$ has a closed form solution 
in terms of the Singular Value Decomposition.  For proofs of
these results, see the proofs of Theorems \ref{thm:closest-point-on-Grassman} and \ref{thm:chordal-mean} in the appendix.

The chordal metric on a Grassman manifold is not a geodesic metric
(that is it is not equal to the length of a shortest geodesic under the
Riemannian metric). However, it is closely related.  In fact, one may
easily show that for $\GRASS{p}{d} \ni \mathcal{X} = \mathrm{span}(\Mat{X})$ and
$\GRASS{p}{d} \ni \mathcal{Y} = \mathrm{span}(\Mat{Y})$
\[
\frac{2}{\pi}\, d_\geo(\mathcal{X}, \mathcal{Y}) \le d_\chord (\mh X, \mh Y) \le d_\geo(\mathcal{X}, \mathcal{Y}) ~.
\]
Furthermore, the path-metric (\cite{Hartley_IJCV_13}) induced by $d_\chord(\cdot, \cdot)$ is
equal to the geodesic distance.


\subsection{Sparse Coding}
\label{sec:sub_sc_grass}

Given a dictionary $\mathbb{D}$
with atoms $\mh D_j\in \PGRASS{p}{d}$ and a query sample 
$\mh X$
the problem of sparse coding can be recast extrinsically as (see
Fig.~\ref{fig:gsc_conceptual} for a conceptual illustration):

\begin{equation}
	l(\mathcal{X},\mathbb{D}) \triangleq
	\underset{\Vec{y}}{\min} \Big\|\mh X - \sum_{j=1}^N \sk y_j
\mh D_j  \Big\|_F^2 + \lambda \|\Vec{y}\|_1.
	\label{eqn:sc_embedded_grass}
\end{equation}

The formulation here varies slightly from the general form given in~\eqref{eqn:grass_coding}, in that the point $\sum_{j=1}^N \sk y_j \mh D_j$ 
does not lie exactly on the manifold $\PGRASS{p}{d}$,
since  
it is not idempotent nor its rank
is necessarily $p$. We call this solution an {\em extrinsic 
solution}; the point coded by the dictionary is
allowed to step out of the manifold.  There is no reason to 
see this as a major flaw, as will be discussed later in 
Section~\ref{subsec:extrinsic}.  

Expanding the Frobenius norm term in~\eqref{eqn:sc_embedded_grass}
results in a convex function in $\Vec{y}$: 

\RIH{
\begin{align*}
  &
    \Bigl\| \Mat{X}\Mat{X}^T -    \sum_{j=1}^{N} \sk y_j
    \Mat{D}_j\Mat{D}_j^T \Bigr\|_F^2  
    = \tr(\Mat{X}^T\Mat{X}\Mat{X}^T\Mat{X})   +  \nonumber\\
	&  \!  \sum_{j,r =1}^{N}{\sk y_j \sk y_r 
	\tr(\Mat{D}_r^T\Mat{D}_j\Mat{D}_j^T\Mat{D}_r)}                          
	-2\sum_{j=1}^{N}{\sk y_j \tr(\Mat{D}_j^T\Mat{X}\Mat{X}^T\Mat{D}_j)}.     
\end{align*}%
}{
\begin{align*}
\Bigl\| \mh X - \sum_{j=1}^{N} \sk y_j \mh D_j \Bigr\|_F^2  
    & = \|\mh X\|_F^2   +  
	 \Bigl\| \sum_{j=1}^{N} \sk y_j \mh D_j \Bigr\|_F^2                        
	-2 \lip \sum_{j=1}^{N} \sk y_j \mh D_j, \mh X \rip ~.    
\end{align*}%
}

The sparse codes can be obtained without explicit embedding of the manifold to 
$\PGRASS{p}{d}$ using $\Pi(\mathcal{X})$. This can be seen by defining
\RIH{ 
$[\mathcal{K}(\Mat{X},\mathbb{D})]_i = \tr(\Mat{D}_i^T\Mat{X}\Mat{X}^T
\Mat{D}_i) = \|\Mat{X}^T\Mat{D}_i\|_F^2$
}
{
$[\mathcal{K}(\Mat{X},\mathbb{D})]_i = \lip \mh X, \mh D_i\rip$
}
as an $N$ dimensional vector storing  the similarity between signal $\Mat{X}$ 
and dictionary atoms in the induced space and
\RIH
{
$[\mathbb{K}(\mathbb{D})]_{i,j} =
 \tr(\Mat{D}_i^T\Mat{D}_j\Mat{D}_j^T\Mat{D}_i) = \|\Mat{D}_i^T\Mat{D}_j\|_F^2$
}{
$[\mathbb{K}(\mathbb{D})]_{i,j} =
 \lip \mh D_i, \mh D_j \rip$
}
as an {$N \times N$} symmetric matrix encoding the similarities between
dictionary atoms (which can be computed offline).
Then, the sparse coding in~\eqref{eqn:sc_embedded_grass} can be
written as:
\begin{equation}
	l(\mathcal{X},\mathbb{D}) = \underset{\Vec{y}}{\min} \:
	\Vec{y}^T \mathbb{K}(\mathbb{D}) \Vec{y}
-2\Vec{y}^T\mathcal{K}(\Mat{X},\mathbb{D})
	+\lambda \|\Vec{y}\|_1\;. 
	\label{eqn:Opt_Grass3}
\end{equation}
The symmetric matrix $\mathbb{K}(\mathbb{D})$ is positive semidefinite since 
for all $\Vec{v} \in \R^N$:
\RIH{
\begin{align*}
	\Vec{v}^T\mathbb{K}(\mathbb{D})\Vec{v} &= \sum_{i=1}^{N}\sum_{j=1}^{N}
v_iv_j\tr(\Mat{D}_i^T\Mat{D}_j\Mat{D}_j^T\Mat{D}_i)\\
	&= \tr \Big(\sum_{i=1}^{N}\sum_{j=1}^{N} v_i
v_j\Mat{D}_i\Mat{D}_i^T\Mat{D}_j\Mat{D}_j^T \Big)\\	
	&= \tr \Big(\sum_{i=1}^{N} v_i \Mat{D}_i\Mat{D}_i^T \sum_{j=1}^{N} v_j
\Mat{D}_j\Mat{D}_j^T\Big)\\ 
	&= \Big\|\sum_{i=1}^{N} v_i \Mat{D}_i\Mat{D}_i^T\Big\|_F^2 \geq 0.
\end{align*}%
}{
\begin{align*}
	\Vec{v}^T\mathbb{K}(\mathbb{D})\Vec{v} &= \sum_{i=1}^{N}\sum_{j=1}^{N}
v_iv_j\lip \mh D_i, \mh D_j \rip 
	= \left< \sum_{i=1}^{N} v_i \mh D_i,\,\, \sum_{j=1}^{N} v_j \mh D_j \right> \\
	&= \Big\|\sum_{i=1}^{N} v_i \mh D_i\Big\|_F^2 \geq 0.
\end{align*}%
} 
Therefore, the  problem is convex and can be efficiently solved using
common packages like CVX~\cite{CVX1,CVX2} or SPAMS~\cite{Mairal_JMLR_2010}.
The problem in~\eqref{eqn:Opt_Grass3} can be transposed into a
vectorized sparse coding problem.
More specifically, let $\Mat{U}\Mat{\Sigma}\Mat{U}^T$ be the SVD of
$\mathbb{K}(\mathbb{D})$.
Then~\eqref{eqn:Opt_Grass3} is equivalent to 

\begin{equation}
	l(\mathcal{X},\mathbb{D}) = \underset{\Vec{y}}{\min} \:
	\Vert \Vec{x}^\ast - \Mat{A}\Vec{y}\Vert^2 +\lambda \|\Vec{y}\|_1, 
	\label{eqn:Opt_Grass4}
\end{equation} 
where $\Mat{A} = \Mat{\Sigma}^{1/2}\Mat{U}^T$ and $\Vec{x}^\ast =
\Mat{\Sigma}^{-1/2}\Mat{U}^T\mathcal{K}(\Mat{X},\mathbb{D})$.
This can be easily verified by plugging $\Mat{A}$ and $\Vec{x}^\ast$ into~\eqref{eqn:Opt_Grass4}. 
Algorithm~\ref{alg:alg_sc} provides the pseudo-code for performing
Grassmann Sparse Coding (gSC).

\begin{algorithm}[!tb]\small
	\SetKwBlock{KwInit}{Initialization.}{}
	\SetKwBlock{KwProc}{Processing.}{}
	\SetAlgoLined	
	\setcounter{AlgoLine}{1}	
	\KwIn
	{Grassmann dictionary $\{ \mathcal{D}_i
\}_{i=1}^N,~\mathcal{D}_i~\in~\GRASS{p}{d}$
	with $\mathcal{D}_i = \mathrm{span}(\Mat{D}_i)$; 
	the query $\GRASS{p}{d} \ni \mathcal{X} = \mathrm{span}(\Mat{X})$
	}
	\KwOut
	{The sparse code {$\Vec{y}^\ast$}
	}
	\BlankLine
	\KwInit{
	\For{$i,j \gets 1$ \KwTo $N$}{
	\mbox{$[\mathbb{K}(\mathbb{D})]_{i,j} \gets \big\|\Mat{D}_i^T\Mat{D}_j
\big\|_F^2$}
	}
	\mbox{$\mathbb{K}(\mathbb{D}) = \Mat{U}\Mat{\Sigma}\Mat{U}^T$}
\tcc{compute SVD of $\mathbb{K}(\mathbb{D})$}
	\mbox{$\Mat{A} \gets \Mat{\Sigma}^{1/2}\Mat{U}^T$}			
	}
	\KwProc{
	\For{$i \gets 1$ \KwTo $N$}{
	\mbox{$[\mathcal{K}(\Mat{X},\mathbb{D})]_i \gets \big\|\Mat{X}^T\Mat{D}_i
\big\|_F^2$}
	}
	\mbox{$\Vec{x}^\ast \gets \Mat{\Sigma}^{-1/2}\Mat{U}^T\mathcal{K}(\Mat{X},
\mathbb{D})$}\\
	$\Vec{y}^\ast \gets \underset{\Vec{y}}{\arg\min} \:	\Vert \Vec{x}^\ast -
\Mat{A}\Vec{y}\Vert^2 +\lambda \|\Vec{y}\|_1$
	}
	\caption{\small Sparse coding on Grassmann manifolds (gSC).}
	\label{alg:alg_sc}
\end{algorithm}

A special case is sparse coding on the Grassmann manifold $\GRASS{1}{d}$, 
which can be seen as a problem on $d-1$ dimensional unit sphere, albeit
with a subtle difference.
More specifically, unlike conventional sparse coding in vector spaces, 
$\Vec{x} \sim -\Vec{x}, \forall \Vec{x} \in \GRASS{1}{d}$,
which results in having antipodals points being equivalent. 
For this special case, the solution proposed in~\eqref{eqn:sc_embedded_grass}
can be understood as sparse coding in the higher dimensional quadratic space,
\ie, $f:\R^d \to \R^{d^2}, f(\Vec{x}) =
[x_1^2,x_1x_2,\cdots,x_d^2]^T$. 
We note that in the quadratic space, $\|f(\Vec{x})\| = 1$ and $f(\Vec{x}) =
f(-\Vec{x})$.

\def \GSC_CONCEPT_SIZE {0.9}
\begin{figure}[!tb]
  \centering
  \begin{subfigure}[b]{0.45 \columnwidth}
  	\centering
  	\includegraphics[width=\GSC_CONCEPT_SIZE
\columnwidth,keepaspectratio]{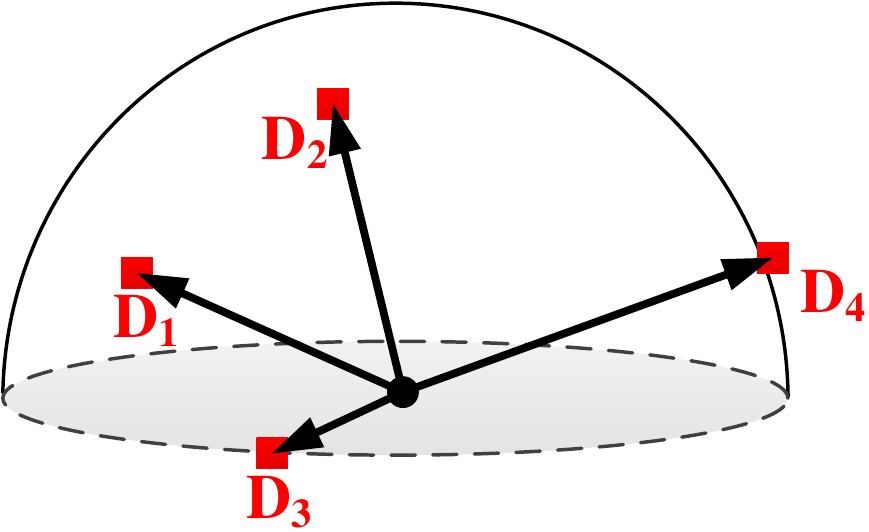} 
  	\caption{}
  \end{subfigure} %
  ~ %
  \begin{subfigure}[b]{0.45 \columnwidth}
  	\centering
  	\includegraphics[width=1.1 \columnwidth,keepaspectratio]{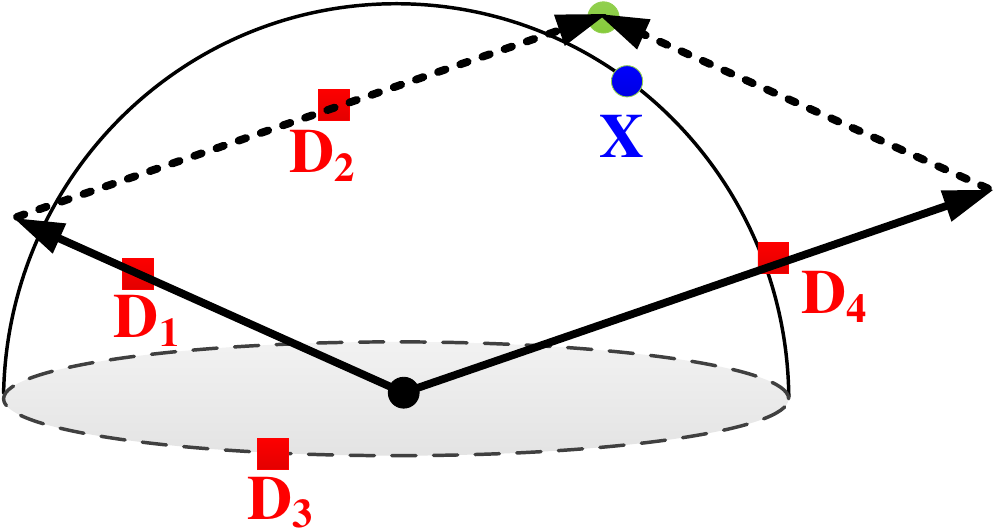} 
  	\caption{}
  \end{subfigure} %
  \caption
    { \small A conceptual diagram of the extrinsic sparse coding addressed
in this work. 
    The hemisphere is being used to represent $\PGRASS{p}{d}$.  Each point on the surface of
    the hemisphere is intended to be a Grassmannian point represented by a
symmetric, idempotent and rank $p$ matrix.   
    {\bf (a)}~.  A Grassmann dictionary on $\PGRASS{p}{d}$ with  four
atoms (red squares).
    {\bf (b)}~. Sparsely describing a query point shown by a blue circle
using dictionary atoms. Here, the combination
    of atoms (green circle) could step out of $\PGRASS{p}{d}$. Having an
 overcomplete dictionary (enough atoms), 
    it is possible to get arbitrarily close to the manifold. This is in
spirit similar to sparse coding in vector spaces.
    More specifically, for a unit norm vector $\Vec{x} \in \R^d$
and a dictionary $\mathbb{D} = \{\Vec{d}_i\}$ with 
    unit norm atoms, the result of sparse coding might be outside the unit
norm sphere in $\R^d$.
    }
  \label{fig:gsc_conceptual}
\end{figure}

\subsection{Locality-Constrained Coding}
\label{sec:sub_llc_grass}

Several studies favor locality  in coding process as
locality could lead to sparsity but not necessarily vice
versa~\cite{Roweis_2000_LLE,Yu:NIPS:2009,Wang_CVPR_2010_LLC}.
In what follows, we describe a coding scheme based on neighborhood
information. We show that the codes with local constraints can be
obtained in closed-form which in turn avoids convex optimization problems
as required for sparse coding. However, there is no free lunch 
here since the new algorithm requires a new parameter, namely the number of
nearest neighbors, to generate the
codes. 

In vector spaces, a fast version of Locality-constrained Linear Coding
(LLC) as proposed by~\citet{Wang_CVPR_2010_LLC}
is described by:

\begin{align}
	&\underset{\Vec{y}}{\min} \,\Vert\Vec{x} -  \Mat{B}\Vec{y}  \Vert^2
\label{eqn:llc_orig}\\
	&\rm{s.t.}~~~\Vec{1}^T\Vec{y} = 1.\nonumber	
\end{align}

In~\eqref{eqn:llc_orig},  $\Vec{x} \in \R^d$ is the query,
$\Mat{B} \in \R^{d \times N_{LC}}$ is a local basis obtained by
simply stacking the $N_{LC}$ nearest neighbors of $\Vec{x}$ from a global
dictionary $\mathbb{D} \in \R^{d \times N}$
and $\Vec{y}$ is the $N_{LC}$ dimensional LLC vector. Recasting the LLC
problem depicted in~\eqref{eqn:llc_orig}
to Grassmann manifolds using the mapping $\Pi(\mathcal{\cdot})$, we obtain:
   
\begin{align}
\begin{split}
\label{eqn:llc_embedded_grass}
	&\underset{\Vec{y}}{\min} \,\Big\|\mh X - \sum_{j=1}^{N_{LC}}
\sk y_j \mh B_j \Big\|_F^2	\\	
	&\rm{s.t.}~~~\Vec{1}^T_{}\Vec{y} = 1.
\end{split}
\end{align}

Observing the constraint
$\Vec{1}^T\Vec{y} = 1$, we may write  

\begin{align}
\begin{split}
	\Big\|\mh X - \sum_{j=1}^{N_{LC}} \sk y_j
\mh B_j \Big\|_F^2 
&= \Big\|\sum_{j=1}^{N_{LC}} \sk y_j (\mh X -\mh B_j) \Big\|_F^2 \\
&= \left< \sum_{i=1}^{N_{LC}}\sk y_i (\mh X - \mh B_i), \, \sum_{j=1}^{N_{LC}}\sk y_j (\mh X - \mh B_j) \right> \\
&= 
\sum_{i, j=1}^{N_{LC}}\sk y_i \sk y_j\left<  \mh X - \mh B_i, \, \mh X - \mh B_j \right> \\
&=	\Vec{y}^T\mathbb{B}\,\Vec{y}
\end{split}
\end{align}
where the elements of matrix  $\mathbb{B}$ are
\begin{align}
\begin{split}
[\mathbb{B}]_{i,j} &= \left<  \mh X - \mh B_i, \, \mh X - \mh B_j \right> \\ 
&= p-\big\|\Mat{X}^T\Mat{B}_i\big\|_F^2-\big\|\Mat{X}^T
\Mat{B}_j\big\|_F^2+\big\|\Mat{B}_j^T\Mat{B}_i\big\|_F^2\;.
\label{eqn:B_gLC}
\end{split}
\end{align}
Then, the minimum in (\ref{eqn:llc_embedded_grass}) may be found 
by solving $\mathbb{B}\hat{\Vec{y}} =
\Vec{1}$, and then rescaling $\hat{\Vec{y}}$ so that it sums to one.
Algorithm~\ref{alg:alg_llc} provides the pseudo-code for performing
Grassmann Locality-constrained Coding (gLC).

A similar formulation albeit intrinsic, for the purpose of nonlinear
embedding of Riemannian manifolds is developed 
by~\citet{Goh_CVPR_2008}. Aside from the different purpose (coding versus embedding), gLC
can exploit an additional codebook learning step (as explained in
\textsection~\ref{sec:dic_learning})
while dictionary learning based on the intrinsic formulation has no
analytic solution.

\def \GLC_CONCEPT_SIZE {0.9}
\begin{figure}[!tb]
  \centering
  \begin{subfigure}[b]{0.45 \columnwidth}
  	\centering
  	\includegraphics[width=\GLC_CONCEPT_SIZE
\columnwidth,keepaspectratio]{concept_gsc1.pdf} 
  	\caption{}
  \end{subfigure} %
  ~ %
  \begin{subfigure}[b]{0.45 \columnwidth}
  	\centering
  	\includegraphics[width=\GLC_CONCEPT_SIZE
\columnwidth,keepaspectratio]{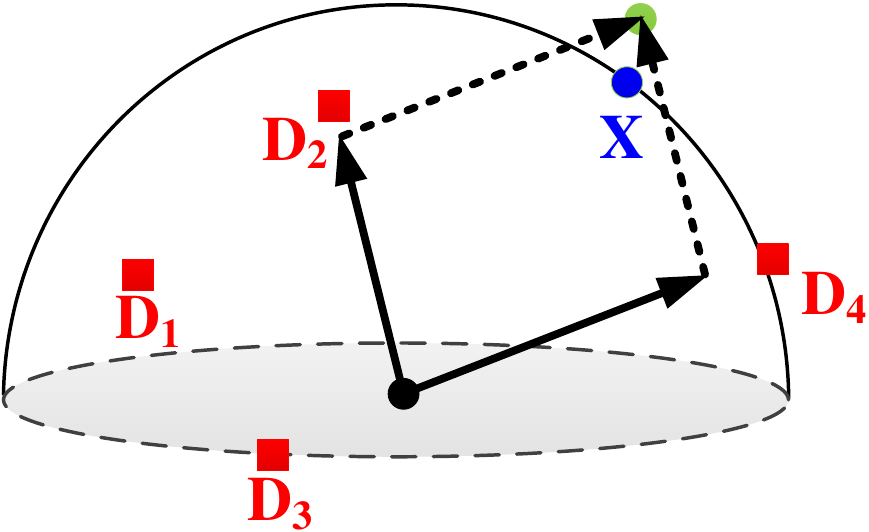} 
  	\caption{}
  \end{subfigure} %
  \caption
    { \small A conceptual diagram of the extrinsic locality constrained
coding on Grassmann manifolds.    
    The hemisphere is being used to represent 
$\PGRASS{p}{d}$.  Each point on the surface of
    the hemisphere is intended to be a Grassmannian point represented by a
symmetric, idempotent and rank $p$ matrix.   
    {\bf (a)}~.  A Grassmann dictionary on $\PGRASS{p}{d}$ with  four
atoms (red squares).
    {\bf (b)}~. Locality constrained coding to describe a query point
$\Mat{X} \in \PGRASS{p}{d}$. 
    Here, only the closest atoms to the query point contribute in coding.
With enough neighbors, 
    it is possible to get arbitrarily close to the manifold.     
    }
  \label{fig:lsc_conceptual}
\end{figure}

\begin{algorithm}[!tb]\small
	\SetKwBlock{KwInit}{Initialization.}{}
	\SetKwBlock{KwProc}{Processing.}{}
	\SetAlgoLined	
	\setcounter{AlgoLine}{1}
	\KwIn	
	{Grassmann dictionary $\{ \mathcal{D}_i
\}_{i=1}^N,~\mathcal{D}_i~\in~\GRASS{p}{d}$
	with $\mathcal{D}_i = \mathrm{span}(\Mat{D}_i)$; 
	the query $\GRASS{p}{d} \ni \mathcal{X} = \mathrm{span}(\Mat{X})$
	}
	\KwOut
	{The gLC {$\Vec{y}^\ast$}
	}
	\BlankLine
	\KwProc{
	\For{$i \gets 1$ \KwTo $N$}{
	\mbox{$\delta_i \gets 2p - 2\big\|\Mat{D}_i^T\Mat{X} \big\|_F^2$}
	}
	\mbox{$\rm{active\_set} \gets$ indexes of the $N_{LC}$ smallest $\delta_i, 1
\leq i \leq N$}\\
	\For{$i \gets 1$ \KwTo $N_{LC}$}{
	\mbox{$\Mat{B}_i \gets \Mat{D}_{\rm{active\_set}(i)}$}
	}
	\For{$i,j \gets 1$ \KwTo $N_{LC}$}{
	\mbox{$[\mathbb{B}]_{i,j} \gets p-\Big\|\Mat{X}^T\Mat{B}_i\Big\|_F^2-
	\Big\|\Mat{X}^T\Mat{B}_j\Big\|_F^2+\Big\|\Mat{B}_j^T\Mat{B}_i\Big\|_F^2$}
	}			
	\mbox{Solve the linear equation system $\mathbb{B}\hat{\Vec{y}} = \Vec{1}$}\\
	\mbox{$\hat{\Vec{y}} \gets \hat{\Vec{y}}/\Vec{1}^T\hat{\Vec{y}}$}\\
	$\Vec{y}^\ast(\rm{active\_set}) \gets \hat{\Vec{y}}$ 	
	}
	\caption{\small Locality-constrained coding on Grassmann manifolds (gLC).}
	\label{alg:alg_llc}
\end{algorithm} 


\subsection{Classification Based on Coding}
\label{sec:sparse_classification}

If the atoms in the dictionary are not labeled
(\eg, if $\mathbb{D}$  is a generic dictionary not tied to any particular
class),
the generated sparse codes (vectors) for both training and query data can
be fed to Euclidean-based classifiers
like support vector machines~\cite{Shawe-Taylor:2004:KMP} for classification.
{
Inspired by the Sparse Representation Classifier (SRC)~\cite{Wright:PAMI:2009},
when the atoms in sparse dictionary $\mathbb{D}$ are labeled, the generated
codes of the query sample can be directly used for classification.
In doing so, let}
\begin{equation*}
\Vec{y}_c = \begin{pmatrix}
\sk y_0\delta \big( l_0-c \big)\\
\sk y_1\delta \big( l_1-c \big)\\
\vdots\\
\sk y_N\delta \big( l_N-c \big)
\end{pmatrix}
\end{equation*}
be the class-specific sparse codes, where $l_j$ is the class label of
atom~$\GRASS{p}{d} \ni \mathcal{D}_j = \mathrm{span}(\Mat{D}_j)$
and $\delta(x)$ is the discrete Dirac function.
An efficient way of utilizing class-specific sparse codes is through
computing residual errors.
In this case, the residual error of query sample $\GRASS{p}{d} \ni
\mathcal{X} = \mathrm{span}(\Mat{X})$ for class $c$ is defined as:

\begin{equation}
    \varepsilon_c(\mathcal{X})= \Big\| \mh X -
\sum\limits_{j=1}^{N} \sk y_j \mh D_j 
    \delta\big( l_j-c \big) \Big\|_F^2\;.
    \label{eqn:sparse_classification}
\end{equation}%

Alternatively, the similarity between query sample $\mathcal{X}$ to class $c$
can be defined as $s(\mathcal{X},c)= h(\Vec{y}_c)$.
The function $h(\cdot)$ could be a linear function like
$\sum\nolimits_{j=1}^{N}\left(\cdot\right)$
or even a non-linear one like $\max\left(\cdot\right)$.
Preliminary experiments suggest that Eq.~\eqref{eqn:sparse_classification}
leads to higher classification accuracies
when compared to the aforementioned alternatives.

\subsection{Extrinsic Nature of the Solution}
\label{subsec:extrinsic}

The solutions proposed in this section (\eg,~\eqref{eqn:sc_embedded_grass}) 
apply to points in $\PGRASS{p}{d}$ and solve coding
extrinsically,
meaning $\sum_{j=1}^N \sk y_{j} \mh D_j$
is not necessarily a point on $\GRASS{p}{d}$.
If, however, it is required that the linear combination of elements
$\sum_j \sk y_j \mh D_j$
actually be used to represent a point on the manifold,
then this can be found by projecting to the closest point
on $\PGRASS{p}{d}$ using Theorem~\ref{thm:chordal-mean} (see appendix).  As shown there, the
resulting point is the weighted chordal mean of the dictionary
atoms $\mh D_j$ with coefficients $[\Vec y]_j$.

Furthermore, the proposed methods follow the general principle of coding in
that the over-completeness of $\mathbb{D}$ will approximate 
$\mh X$, and $\sum_j \sk y_j \mh D_j$ can be
expected to be closely adjacent to a Grassmann point. 
An intrinsic version of~\eqref{eqn:sc_embedded_grass} can be written as:
\begin{equation}
l_{\calG}(\mathcal{X},\mathbb{D}) \triangleq
    \underset{\Vec{y}}{\min} \:
    \Bigl\| \mh X-
    \mathrm{Proj}\Big(\sum_{j=1}^{N} \sk y_{j}
\mh D_j\Big) \Bigr\|_F^2
    +\lambda \|\Vec{y}\|_1, 
    \label{eqn:rene_equ}
\end{equation}%
where $\mathrm{Proj}(\sum_{j=1}^{N} \sk y_{j}
\mh D_j)$ is the weighted chordal mean of the dictionary
atoms, as shown by Theorem \ref{thm:chordal-mean}.  This formulation is precisely
of the form~\eqref{eqn:grass_coding}, where the coding is the weighted 
chordal mean.
The involvement of SVD makes solving~\eqref{eqn:rene_equ} challenging.  
While seeking efficient ways of solving~\eqref{eqn:rene_equ}
is interesting, it is beyond the scope of this work.
The coding error given by~\eqref{eqn:rene_equ} and~\eqref{eqn:sc_embedded_grass} will normally be very close, 
making~\eqref{eqn:sc_embedded_grass} an efficient compromise solution.


\section{Dictionary Learning}
\label{sec:dic_learning}

Given a finite set of observations 
$\mathbb{X} = \left \{ \mathcal{X}_i  \right \}_{i=1}^{m}, \: \GRASS{p}{d}
\ni \mathcal{X}_i =\mathrm{span}(\Mat{X}_i)$,
the problem of dictionary learning on Grassmann manifolds is defined as
minimizing the following cost function:
\begin{equation}
	h(\mathbb{D}) \triangleq \sum\limits_{i=1}^{m}
l_\calG(\mathcal{X}_i,\mathbb{D}),
    \label{eqn:grass_dic_learning}
\end{equation}
with $\mathbb{D} = \left \{ \mathcal{D}_j  \right \}_{j=1}^{N}, \:
\GRASS{p}{d} \ni \mathcal{D}_j =\mathrm{span}(\Mat{D}_j)$
being a dictionary of size $N$.
Here, $l_\calG(\mathcal{X},\mathbb{D})$ is a loss function
and should be small if $\mathbb{D}$ is ``good'' at representing $\mathcal{X}$.
In the following text, we elaborate on how a Grassmann dictionary can be learned.



Aiming for sparsity, the $ \ell_1 $-norm regularization is usually
employed to obtain the most common form of
$l_\calG(\mathcal{X},\mathbb{D})$ as depicted 
in Eq.~\eqref{eqn:sc_embedded_grass}. 
With this choice, the problem of dictionary learning on Grassmann manifolds
can be written as:
\begin{equation}
\underset{ \{\Vec{y}_i \}_{i=1}^m, \mathbb{D}}{\min} \:
    \sum_{i=1}^{m}\Big\| \mh X_i-
\sum_{j=1}^{N} \sk{y_i}_j \mh D_j \Big\|_F^2
    +\lambda 
    \sum_{i=1}^{m} \|\Vec{y}_i\|_1.
    \label{eqn:dic_learn_grass_orig_equ}
\end{equation}
Due to the non-convexity of~\eqref{eqn:dic_learn_grass_orig_equ} and
inspired by the solutions
in Euclidean spaces, we propose to solve~
\eqref{eqn:dic_learn_grass_orig_equ} by alternating
between the two sets of variables, $\mathbb{D}$ and $\{\Vec{y}_i \}_{i=1}^m$.
More specifically, minimizing~\eqref{eqn:dic_learn_grass_orig_equ} over
sparse codes $\Vec{y}$
while dictionary $\mathbb{D}$ is fixed is a convex problem. Similarly,
minimizing the overall problem
over $\mathbb{D}$ with fixed $\{\Vec{y}_i \}_{i=1}^m$ is convex as well.

Therefore, to update dictionary atoms we break the minimization problem
into $N$ sub-minimization problems by independently updating each atom, $\mh D_r$,
in line with general practice in dictionary learning~\cite{ELAD_SR_BOOK_2010}.
%
%
%
To update $\mh D_r$, we write 
\begin{equation}
\sum_{i=1}^{m}\,\Big\| \mh X_i-
\sum_{j=1}^{N}\, \sk{y_i}_j \mh D_j \Big\|_F^2
    = \sum_{i=1}^{m}\,\Big\|\Big(\mh X_i-\sum_{j\ne r}\, \sk{y_i}_j \mh D_j\Big) -  \sk{y_i}_r \mh D_r \Big\|_F^2 ~.
    \label{eqn:dic_learn_grass_orig_equ_2}
\end{equation}
All other terms in (\ref{eqn:dic_learn_grass_orig_equ}) being 
independent of $\mh D_r$, and 
since $\|\mh D_r\|_F^2 = p$ is fixed, minimizing this with respect to $\mh D_r$
is equivalent to minimizing $\mathcal{J}_r = -2\lip \Mat{S}_r,\, \mh D_r \rip$
where
\begin{align}
\begin{split}
\label{eqn:s_r_dl}
\Mat{S}_r =  
\sum_{i=1}^m \,\sk{y_i}_r\, \Big(   
     \mh X_i - 
    \sum_{j \neq r} \sk{y_i}_j \mh D_j\Big) .
\end{split}
\end{align}
Finally, minimizing $\mathcal{J}_r = -2\lip \Mat{S}_r,\, \mh D_r \rip$
is the same as minimizing $\| \Mat{S}_r - \mh D_r\|$ over
$\mh D_r$ in $\PGRASS{n}{p}$.  The solution to this problem is
given by projecting $\Mat{S}_r$ onto the manifold, using
Theorem \ref{thm:closest-point-on-Grassman} in the appendix.

%
%
Algorithm~\ref{alg:gdl_pseudo_code} details the pseudo-code for learning a
dictionary on Grassmann manifolds.
Fig.~\ref{fig:learnt_atoms_example} shows examples of a ballet dance
and the atoms learned by the proposed method.
From each atom, we plot the dominant eigendirection because it is visually
more informative.
Note that the learned atoms capture the ballerina movements.

To perform coding, we have relaxed the idempotent and rank constraints of
the mapping $\Pi(\cdot)$ since
matrix addition and subtraction do not preserve these constraints.
However, for dictionary learning, the orthogonality constraint ensures the
dictionary atoms have the required structure.
Before concluding this section, we note that dictionary learning for gLC follows verbatim. The only difference from what we have developed
in Algorithm~\ref{alg:gdl_pseudo_code} is the coding step which is done using Algorithm~\ref{alg:alg_llc}.


\def \DICSIZE {0.20}
\begin{figure}[!t]
\begin{minipage}{1\columnwidth}
      \begin{minipage}{0.025\columnwidth}
         \centering
         \small
         {\bf (a)}
       \end{minipage}
      \hfill
      \begin{minipage}{0.95\columnwidth}
      \centering
        \includegraphics[width=\DICSIZE\columnwidth,keepaspectratio]
{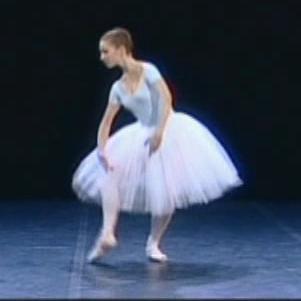}
        \includegraphics[width=\DICSIZE\columnwidth,keepaspectratio]
{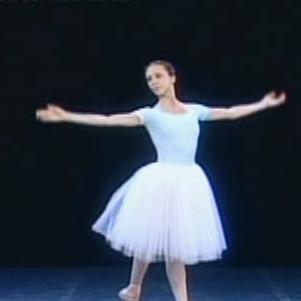}
        \includegraphics[width=\DICSIZE\columnwidth,keepaspectratio]
{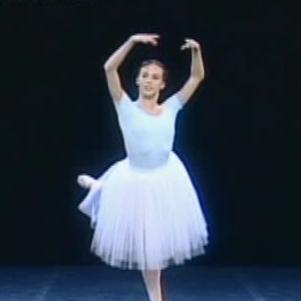}
        \includegraphics[width=\DICSIZE\columnwidth,keepaspectratio]
{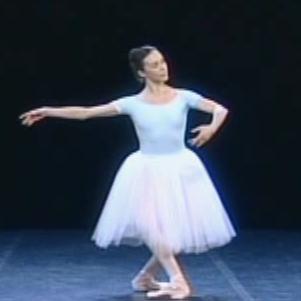}
      \end{minipage}
      \vspace{0.5ex}
      \begin{minipage}{0.025\columnwidth}
         \centering
         \small
         {\bf (b)}
       \end{minipage}
       \hfill
      \begin{minipage}{0.95\columnwidth}
      \centering
        \includegraphics[width=\DICSIZE\columnwidth,keepaspectratio]
{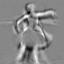}
        \includegraphics[width=\DICSIZE\columnwidth,keepaspectratio]
{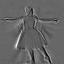}
        \includegraphics[width=\DICSIZE\columnwidth,keepaspectratio]
{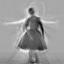}
        \includegraphics[width=\DICSIZE\columnwidth,keepaspectratio]
{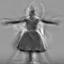}
      \end{minipage}
\end{minipage}
      \vspace{0.5ex}
  \caption
    {
    {\bf (a)}~Examples of actions performed by a ballerina.
    {\bf (b)}~The dominant eigenvectors for four atoms learned by the
proposed Grassmann Dictionary Learning (gDL) method
    (grayscale images were used in gDL).
    }
  \label{fig:learnt_atoms_example}
\end{figure}

\begin{algorithm}[!tb]
\footnotesize
\SetKwBlock{KwInit}{Initialization.}{}
	\SetKwBlock{KwProc}{Processing.}{}
	\SetAlgoLined	
	\setcounter{AlgoLine}{1}
	\KwIn{training set {$\mathbb{X} \mbox{=} \left\{  \mathcal{X}_i
\right\}_{i=1}^{m}$},
	where each {$\GRASS{p}{d} \ni \mathcal{X}_i = \mathrm{span}(\Mat{X}_i)$};
	$nIter$: number of iterations
	}
	\KwOut{
	Grassmann dictionary \mbox{$\mathbb{D}= \left\{
\mathcal{D}_i\right\}_{i=1}^{N}$, where
	 $\GRASS{p}{d} \ni  \mathcal{D}_i = \mathrm{span}(\Mat{D}_i)$}
	}
	\BlankLine
	\KwInit{	
	Initialize the dictionary $\mathbb{D}$ by selecting $N$ samples from
$\mathbb{X}$ randomly
	}
	\KwProc{
	\For{ $t = 1$ \KwTo $nIter$}{
	\tcp{Sparse Coding Step using Algorithm~\ref{alg:alg_sc}}%
	\For{ $i = 1$ \KwTo $m$}{
		\mbox{$\Vec{y}_i \gets 
		\underset{\Vec{y}}{\min} \Big\|\mh X_i - \sum_{j=1}^N\limits
		[\Vec{y}]_j \mh D_j  \Big\|_F^2 + 		\lambda \|\Vec{y}\|_1$}
	}
	\tcp{Dictionary update step}%
	\For{ $r = 1$ \KwTo $N$}{
	\mbox{Compute $\Mat{S}_r$ according to Eq.~\eqref{eqn:s_r_dl}.}\\
	\mbox{$\{{\lambda}_k,{\Vec{v}}_k\} \gets
		\mbox{ eigenvalues and eigenvectors of } \Mat{S}_r$}\\
	\mbox{$\Mat{S}_r\Vec{v} = \lambda \Vec{v}; \lambda_1 \geq \lambda_2 
	\geq \cdots \geq \lambda_d$}\\
	\mbox{$\Mat{D}_r^\ast \gets [\Vec{v}_1|\Vec{v}_2|\cdots|\Vec{v}_p]$}
	}
	}
	}
	\caption{Grassmann Dictionary Learning (gDL)}
	\label{alg:gdl_pseudo_code}
\end{algorithm}


\section{Kernelized Coding and Dictionary Learning on Grassmann Manifolds}
\label{sec:kernel_coding}

In this section, we are interested in coding and dictionary learning on
higher-dimensional (possibly infinite-dimensional) Grassmann manifolds.
Such treatment can be helpful in dealing with non-linearity of data 
since one can hope higher-dimensional manifolds diminish non-linearity.
This follows the practice of using higher dimensional spaces in vector
spaces~\cite{Shawe-Taylor:2004:KMP}.
To this end, we make use of a mapping {$\phi : \R^d \rightarrow
\mathcal{H}$}
from $\R^d$ into a Reproducing Kernel Hilbert Space (RKHS)
$\mathcal{H}$
with a real-valued kernel function $k(\cdot,\cdot)$
on {$\R^d \times \R^d$},
such that {$\forall \Vec{x},\Vec{x}^\prime \in \R^d,\;
\langle \phi(\Vec{x}),\phi(\Vec{x}^\prime)\rangle = \phi(\Vec{x})^T
\phi(\Vec{x}^\prime) = k(\Vec{x},\Vec{x}^\prime)$}
\cite{Shawe-Taylor:2004:KMP}. 

Our goal here is to perform both coding and dictionary learning in
$\mathcal{H}$,
but for efficiency we want to avoid explicitly working in $\mathcal{H}$.
In other words, we would like to obtain the solutions by only using
$k(\cdot,\cdot)$.
In the following text we show how this can be achieved.

\subsection{Kernel Coding}
\label{sec:sub_kernel_sparse_coding}

Let {$\Psi(\Mat{X}) = [\Vec{\psi}_{1}|\Vec{\psi}_{2}|\cdots|\Vec{
\psi}_{p}]$} be an orthonormal basis of order $p$
for the column space of $\Phi(\Mat{X}) = [
\phi(\Vec{x}_1)|\phi(\Vec{x}_2)|\cdots|\phi(\Vec{x}_q)],\; p \leq q$ in
$\mathcal{H}$. 
The $q \times q$ Gram matrix $\Phi(\Mat{X})^T\Phi(\Mat{X})$ whose $i$-th
row and $j$-th column entry is 
$k(\Vec{x}_i,\Vec{x}_j)$ can be decomposed as:
\begin{equation}
	\Phi(\Mat{X})^T \Phi(\Mat{X}) = \Mat{U}_\Mat{X}\Mat{\Sigma}_\Mat{X}
\Mat{U}_\Mat{X}^T.
	\label{eqn:Hilbert1}
\end{equation}
The connection between $\Mat{U}_\Mat{X}$ and $\Mat{\Psi}_\Mat{X}$ is 
a ``trick'' used to compute the principal components of a matrix
that has considerably less columns than rows~\cite{Turk_Eigenfaces},
and can be easily established by picking the $p$ largest singular values of
$\Mat{\Sigma}_\Mat{X}$ 
and corresponding elements of $\Mat{U}_\Mat{X}$ (denoted by $\downarrow$
below) as:
\begin{equation}
	\Psi(\Mat{X}) = \Mat{\Phi}_\Mat{X} \Mat{U}_{\Mat{X}\downarrow}
\Mat{\Sigma}_{\Mat{X}\downarrow}^{-1/2}.
    \label{eqn:Hilbert2}
\end{equation}%

The sparse coding problem on a Grassmann manifold embedded in $\mathcal{H}$ 
can be understood as the kernel version of~\eqref{eqn:sc_embedded_grass}, as
depicted below:
\noindent
\begin{align}
	\underset{\Vec{y}}{\min} \:
    \Bigl \| \mh \Psi(\Mat{X}) -
    \sum_{j=1}^{N} \sk y_{j} \mh \Psi(\Mat{D}_j)
    \Bigr\|_F^2    +\lambda \|\Vec{y}\|_1.
    \label{eqn:KGDL_Grass}
\end{align}%

A similar statement to what we have in \textsection~\ref{sec:sub_sc_grass}
holds here for the convexity of~\eqref{eqn:KGDL_Grass}. 
Therefore, sparse codes can be obtained if the Frobenius inner products between
$\Psi(\Mat{X})$ and elements of the dictionary, \ie, $\lbrace
\Psi({\Mat{D}_i}) \rbrace_{i=1}^{N}$ are known.
Given $\lbrace \Vec{z}_{i} \rbrace_{i=1}^{q_\Mat{Z}},\;  \Vec{z}_{i} \in
\R^d$ and 
$\lbrace \Vec{x}_{i} \rbrace_{i=1}^{q_\Mat{X}},\;  \Vec{x}_{i} \in
\R^d$,
the Frobenius inner product between the corresponding subspaces
$\Psi(\Mat{Z})$ and $\Psi(\Mat{X})$ in $\mathcal{H}$ can be obtained as:
\begin{align}
    \big \langle \Psi(\Mat{Z}), \Psi(\Mat{X}) \big\rangle
    &= \tr \big(\Mat{\Psi}_\Mat{Z}^T \Mat{\Psi}_\Mat{X} \big) \notag\\
    &= \tr \Big( \Mat{\Sigma}_{\Mat{Z}\downarrow}^{-1/2}
\Mat{U}_{\Mat{Z}\downarrow}^T \Phi(\Mat{Z})^T
    \Phi(\Mat{X}) \Mat{U}_{\Mat{X}\downarrow}\Mat{\Sigma}_{\Mat{X}
\downarrow}^{-1/2}
    \Big )\notag\\
    & = \tr \Big (\Mat{\Sigma}_{\Mat{Z}\downarrow}^{-1/2}
\Mat{U}_{\Mat{Z}\downarrow}^T 
    \Mat{K}(\Mat{Z}, \Mat{X}) \Mat{U}_{\Mat{X}\downarrow}\Mat{\Sigma}_{
\Mat{X}\downarrow}^{-1/2}
    \Big),
    \label{eqn:Hilbert3}
\end{align}%
where $\Mat{K}(\Mat{Z}, \Mat{X})$ is a $q_\Mat{Z} \times q_\Mat{X}$ matrix
where
$i$-{th} row and $j$-{th} column entry is $k(\Vec{z}_i,\Vec{x}_j)$.
Therefore, a similar approach to \textsection~\ref{sec:sub_sc_grass} 
can be employed to obtain the sparse codes in~\eqref{eqn:KGDL_Grass}.
Algorithm~\ref{alg:alg_ksc} provides the pseudo-code for performing kernel
sparse coding on Grassmann manifolds 
(kgSC).

\begin{algorithm}[!tb]\small
	\SetKwBlock{KwInit}{Initialization.}{}
	\SetKwBlock{KwProc}{Processing.}{}
	\SetAlgoLined	
	\setcounter{AlgoLine}{1}	
	\KwIn
	{Dictionary $\{ \Mat{D}_i \}_{i=1}^N, \R^{d \times q_i} \ni
\Mat{D}_i = \{ \Vec{d}_{i,j} \}_{j=1}^{q_i}$; 
	the query sample $\Mat{X}\in~\R^{d \times q}$
	}
	\KwOut
	{The sparse code {$\Vec{y}^\ast$}
	}
	\BlankLine
	\KwInit{
	\For{$i \gets 1$ \KwTo $N$}{
	\mbox{$[\Mat{K}(\Mat{D}_i)]_{j,l} \gets k(\Vec{d}_{i,k},\Vec{d}_{i,l})$}\\
	\mbox{$\Mat{K}(\Mat{D}_i) = \Mat{U}_{\Mat{D}_i}\Mat{\Sigma}_{\Mat{D}_i}
\Mat{U}_{\Mat{D}_i}^T$} 
	\tcp*{SVD of $\Mat{K}(\Mat{D}_i)$}	
	}
	\For{$i,j \gets 1$ \KwTo $N$}{
	\mbox{$[\mathbb{K}(\mathbb{D})]_{i,j} \gets
\Big\|\Mat{\Sigma}_{\Mat{D}_j\downarrow}^{-\frac{1}{2}}
\Mat{U}_{\Mat{D}_j\downarrow}^T 
	\Mat{K}({\Mat{D}_j}, \Mat{D}_i) \Mat{U}_{\Mat{D}_i\downarrow}\Mat{
\Sigma}_{\Mat{D}_i\downarrow}^{-\frac{1}{2}} \Big\|_F^2$}	
	}
	\mbox{$\mathbb{K}(\mathbb{D}) = \Mat{U}\Mat{\Sigma}\Mat{U}^T$}
\tcp{compute SVD of $\mathbb{K}(\mathbb{D})$}
	\mbox{$\Mat{A} \gets \Mat{\Sigma}^{1/2}\Mat{U}^T$}
	}
	\KwProc{	
	\mbox{$[\Mat{K}(\Mat{X})]_{i,j} \gets k(\Vec{x}_{i},\Vec{x}_{j})$}\;
	\mbox{$\Mat{K}(\Mat{X}) = \Mat{U}_{\Mat{X}}\Mat{\Sigma}_{\Mat{X}}\Mat{U}_{
\Mat{X}}^T$} 
	\tcp*{SVD of $\Mat{K}(\Mat{X})$}	
	\For{$i \gets 1$ \KwTo $N$}{	
	\mbox{$[\Mat{K}(\Mat{X},\Mat{D}_i)]_{j,l} \gets k(\Vec{x}_j,\Vec{d}_{i,l})$}\\
	\mbox{$[\mathcal{K}(\Mat{X},\mathbb{D})]_i \gets
\Big\|\Mat{\Sigma}_{\Mat{X}\downarrow}^{-\frac{1}{2}}
\Mat{U}_{\Mat{X}\downarrow}^T 
	\Mat{K}(\Mat{X}, \Mat{D}_i) \Mat{U}_{\Mat{D}_i\downarrow}\Mat{\Sigma}_{
\Mat{D}_i\downarrow}^{-\frac{1}{2}} \Big\|_F^2$}	
	}
	\mbox{$\Vec{x}^\ast \gets \Mat{\Sigma}^{-1/2}\Mat{U}^T\mathcal{K}(\Mat{X},
\mathbb{D})$}\\
	$\Vec{y}^\ast \gets \underset{\Vec{y}}{\arg\min} \:	\Vert \Vec{x}^\ast -
\Mat{A}\Vec{y}\Vert^2 +\lambda \|\Vec{y}\|_1$
	}
	\caption{\small Kernel sparse coding on Grassmann manifolds (kgSC).}
	\label{alg:alg_ksc}
\end{algorithm} 


Based on the development in~\textsection\ref{sec:sub_llc_grass},
the kernel version of gLC algorithm on Grassmann manifolds or kgLC for
short can be obtained by
computing the elements of matrix $\mathbb{B}$ in $\mathcal{H}$, \ie,
kernelizing Eq.\eqref{eqn:B_gLC}.
That is,  
%
%
%
\begin{align}
[\mathbb{B}_\mathcal{H}]_{i,j} &= 
p-\Big\|\Psi(\Mat{X})^T\Psi({\Mat{B}_i})\Big\|_F^2-
\Big\|\Psi(\Mat{X})^T\Psi({\Mat{B}_j})\Big\|_F^2+
\Big\|\Psi({\Mat{B}_j})^T\Psi({\Mat{B}_i})\Big\|_F^2 \notag\\
&= p - 
\Big\|\Mat{\Sigma}_{\Mat{B}_i\downarrow}^{-1/2}
\Mat{U}_{\Mat{B}_i\downarrow}^T 
\Mat{K}(\Mat{B}_i, \Mat{X}) \Mat{U}_{\Mat{X}\downarrow}\Mat{\Sigma}_{
\Mat{X}\downarrow}^{-1/2}\Big\|_F^2
\notag\\&\hspace{4ex}-
\Big\|\Mat{\Sigma}_{\Mat{B}_j\downarrow}^{-1/2}
\Mat{U}_{\Mat{B}_j\downarrow}^T 
\Mat{K}(\Mat{B}_j, \Mat{X}) \Mat{U}_{\Mat{X}\downarrow}\Mat{\Sigma}_{
\Mat{X}\downarrow}^{-1/2}\Big\|_F^2
\notag\\&\hspace{4ex}+
\Big\|\Mat{\Sigma}_{\Mat{B}_i\downarrow}^{-1/2}
\Mat{U}_{\Mat{B}_i\downarrow}^T \Mat{K}(\Mat{B}_i, \Mat{B}_j) 
\Mat{U}_{\Mat{B}_j\downarrow}\Mat{\Sigma}_{\Mat{B}_j\downarrow}^{-1/2}\Big
\|_F^2.
\label{eqn:llc_kernel_sol}
\end{align}

Once we have $\mathbb{B}_\mathcal{H}$ at our disposal, the codes are
obtained by 
first solving the linear system of equations,
$\mathbb{B}_{\mathcal{H}}\hat{\Vec{y}}~=~\Vec{1}$,
and then rescaling the result to have unit $\ell_1$ norm.

\subsection{Kernel Dictionary Learning}
\label{sec:kernel_dic_learning_for_sc}

Given a finite set of observations 
$\mathbb{X} = \left \{ \Mat{X}_i  \right \}_{i=1}^{m}, \: \R^{d \times q_i} \ni \Mat{X}_i = \{\Vec{x}_{i,j}\}_{j=1}^{q_i}$,
the problem of kernel dictionary learning in the RKHS $\mathcal{H}$ 
can be written by kernelizing~\eqref{eqn:dic_learn_grass_orig_equ} as:

\begin{align}
\underset{ \{\Vec{y}_i \}_{i=1}^m, \mathbb{D}}{\min} \:
    &\sum\limits_{i=1}^{m}\bigg\| \mh \Psi(\Mat{X}_i) - \sum\limits_{j=1}^{N} 
    [\Vec{y}_i]_j \mh \Psi(\Mat{D}_j) \bigg\|_F^2
    +\lambda     \sum\limits_{i=1}^{m} \|\Vec{y}_i\|_1.
    \label{eqn:kernel_dic_learn_grass_orig_equ}
\end{align}

Similar to the linear case, the dictionary in $\mathcal{H}$ 
is updated atom by atom (\ie, atoms are assumed to be independent) by 
fixing the codes $\{\Vec{y}_i\}_{i=1}^m$.
First, we note that a basis in $\mathcal{H}$ can be written as a linear combination of 
its samples. For $\Psi(\Mat{X})$ and as shown in Eq.~\eqref{eqn:Hilbert2}, 
$\Psi(\Mat{X}) = \Phi(\Mat{X}) \Mat{A}_{\Mat{X}}$, where $\Mat{A}_{\Mat{X}}$ is obtained from $\Mat{K}(\Mat{X},\Mat{X})$. 
Similarly, 
\begin{equation}
	\Psi(\Mat{D}_r) = \Phi(\Mat{D}_r) \Mat{A}_{r} = \Phi(\bigcup_{i}\Mat{X}_i) \Mat{A}_{r}, \; [\Vec{y}_i]_r \neq 0\;.
	\label{eqn:psi_dr}
\end{equation} 
As such, $\Psi(\Mat{D}_r)$ is fully determined if
$\Mat{A}_{r}$ is known as $\Mat{K}(\cdot,\Mat{D}_r) = \Mat{K}(\cdot,\bigcup_{i}\Mat{X}_i), \; [\Vec{y}_i]_r \neq 0$.
The orthogonality constraint for $\Psi(\Mat{D}_r)$ can be written as:
\begin{equation}
    \Psi(\Mat{D}_r)^T\Psi(\Mat{D}_r) = 
    \Mat{A}_{r}^T \Mat{K}(\Mat{D}_r,\Mat{D}_r)\Mat{A}_{r}
    = \mathbf{I}_p,
    \label{eqn:KGDL_orth_condition}
\end{equation}%
where $\Mat{K}(\Mat{D}_r,\Mat{D}_r) = \Mat{K}(\bigcup_{i}\Mat{X}_i,\bigcup_{i}\Mat{X}_i), \; \sk{y_i}_r \neq 0$.
%
Following similar steps to what developed in~\textsection~\ref{sec:dic_learning}, to obtain $\Psi(\Mat{D}_r)$ 
we need to maximize $\tr (\Psi(\Mat{D}_r)^T \Gamma \Psi(\Mat{D}_r))$
by taking the orthogonality constraint (\ie, Eq.~\eqref{eqn:KGDL_orth_condition}) into account. Here $\Gamma$ is the kernel form 
of~\eqref{eqn:s_r_dl} written as:
\begin{align}
\begin{split}
\label{eqn:s_r_kdl}
\Gamma =  
\sum_{i=1}^m \,\sk{y_i}_r\, \Big(   
     \mh \Psi(X_i) - 
    \sum_{j \neq r} \sk{y_i}_j \mh \Psi(D_j)\Big) .
\end{split}
\end{align}
Defining 
\begin{equation}
	\Mat{B}(\Mat{X},\Mat{Z}) = \Mat{K}(\Mat{X},\Mat{Z})\Mat{A}_{\bf Z}\Mat{A}_{\bf Z}^T \Mat{K}(\Mat{Z},\Mat{X}),
	\label{eqn:kgdl_B_matrix}
\end{equation}
and 
\begin{equation}
    \Mat{S}_r^\Psi = \sum_{i=1}^{m}{[\Vec{y}_i]_{r} \Bigl( \Mat{B}(\Mat{D}_r,\Mat{X}_i) -
    \sum_{\substack{j=1\\j \neq r}}^{N}{[\Vec{y}_i]_{j}\Mat{B}(\Mat{D}_r,\Mat{D}_j)}  
    \Bigr)}, 
    \label{eqn:kgdl_sol}
\end{equation}%
maximizing $\tr (\Psi(\Mat{D}_r)^T\Gamma \Psi(\Mat{D}_r))$ with the orthogonality constraint boils down to:
\begin{align}
    \Mat{A}_r^\ast &= \underset{\Mat{A}_r}{\operatorname{argmax}} \;
    \tr (\Mat{A}_r^T \Mat{S}_r^\Psi \Mat{A}_r),\notag\\  
    &\text{s.t.} ~~ \Mat{A}_r^T \Mat{K}(\Mat{D}_r,\Mat{D}_r) \Mat{A}_r =    \mathbf{I}_p.
    \label{eqn:GDL_opt1}
\end{align}%


%
The solution of the above problem is given by the 
leading eigenvectors of the generalized eigenvalue problem
$\Mat{S}_r^\Psi \Vec{v} = \lambda \Mat{K}(\Mat{D}_r,\Mat{D}_r) \Vec{v}$~\cite{Kokiopoulou_2011}.
In practice one might want to pick 
a small number of $\Mat{X}_i$ that contributed more dominantly to Eq.~\eqref{eqn:psi_dr} to describe $\Psi(\Mat{D}_r)$ and
hence reduce the computational load of dictionary learning.
The steps of determining the kernel dictionary for Grassmann manifolds (kgDL) are shown in Algorithm~\ref{alg:kgdl_pseudo_code}.

\begin{algorithm}[!tb]
\footnotesize
\SetKwBlock{KwInit}{Initialization.}{}
	\SetKwBlock{KwProc}{Processing.}{}
	\SetAlgoLined	
	\setcounter{AlgoLine}{1}	
	\KwIn
	{training set {$\mathbb{X} \mbox{=} \left\{  \Mat{X}_i \right\}_{i=1}^{m}$},
	where {$\Mat{X}_i \in \R^{d \times q_i}$};
	$k(\cdot,\cdot)$, a kernel function;
	$nIter$, the number of iterations.
	}
	\KwOut{
	Grassmann dictionary represented by \mbox{$\mathbb{A}= \left\{  \Mat{A}_{D_i} \right\}_{i=1}^{N}$}
	and \mbox{$\mathbb{K}= \left\{  \Mat{K}(\cdot,\Mat{D}_i) \right\}_{i=1}^{N}$}
	}
	\BlankLine
	\KwInit{	
	\For{$i \gets 1$ \KwTo $m$}{
	\mbox{$[\Mat{K}(\Mat{X}_i)]_{j,l} \gets k(\Vec{x}_{i,k},\Vec{x}_{i,l})$}\\
	\mbox{$\Mat{K}(\Mat{X}_i) = \Mat{U}_{\Mat{X}_i}\Mat{\Sigma}_{\Mat{X}_i}\Mat{U}_{\Mat{X}_i}^T$} 
	\tcp*{SVD of $\Mat{K}(\Mat{D}_i)$}			
	\mbox{$ \Mat{A}_{X_i} \gets \Mat{U}_{\Mat{X}_i\downarrow}\Mat{\Sigma}_{\Mat{X}_i\downarrow}^{-1/2}$}
	}
	Initialize the dictionary {$\mathbb{D}= \left\{  \Mat{D}_i \right\}_{i=1}^{N}$}
	by selecting $N$ samples from $\mathbb{X}$ randomly\;
	\mbox{$\Mat{\Gamma} \gets \Mat{0}_{N \times N}$}\;
	\mbox{$\Vec{\gamma} \gets \Mat{0}_{N \times 1}$}\;
	}
	\KwProc{
	\For{ $t = 1$ \KwTo $nIter$}{
	\tcp{Coding Step}%
	\mbox{Use kgSC or kgLC algorithms to obtain the codes $\Vec{y}_i$}\; 
	\tcp{Dictionary update step}%
	\For{ $r = 1$ \KwTo $N$}{
	\mbox{$\Mat{K}(\cdot,\Mat{D}_r) \gets \Mat{K}(\cdot,\bigcup_{i,[\Vec{y}_{i}]_r \neq 0}\Mat{X}_i)$}\;
	\mbox{Compute $\Mat{S}_r^\Psi$ according to Eq.~\eqref{eqn:kgdl_sol}.}\;
	\mbox{$\{\lambda_k,\Vec{v}_k\} \gets$
	 generalized eigen(values/vectors) of} $\Mat{S}_r^\Psi{\Vec{v}} = 
	 \lambda \Mat{K}(\Mat{D}_r,\Mat{D}_r)\Vec{v};
	\;\; \lambda_1 \geq \lambda_2 \geq \cdots \geq \lambda_d$\;
	\mbox{$\Mat{A}_{r}^\ast \gets [{\Vec{v}}_1|{\Vec{v}}_2|\cdots|{\Vec{v}}_p]$}\;
	}
	}
	}
	\caption{\small Kernelized Grassmann Dictionary~Learning~(kgDL)}
	\label{alg:kgdl_pseudo_code}
\end{algorithm} 

\section{Experiments}
\label{sec:experiments}

Two sets of experiments\footnote{Matlab codes are available at
\mbox{\url{https://sites.google.com/site/mehrtashharandi/}}}
are presented in this section.
In the first set of experiments, we evaluate the performance of the
proposed coding methods (as described in~\textsection~\ref{sec:Coding}
and~\textsection~\ref{sec:kernel_coding}) without dictionary learning. 
This is to contrast proposed coding schemes to previous state-of-the-art
methods on several popular closed-set classification tasks.
To this end, each point in the training set is considered as an atom in the
dictionary.
Since the atoms in the dictionary are labeled in this case, the residual
error approach for classification 
(as described in~\textsection~\ref{sec:sparse_classification} and~\textsection~\ref{sec:kernel_coding}) 
will be used to determine the label of a query point.
In the second set of experiments, the performance of the coding methods is
evaluated
in conjunction with the proposed dictionary learning algorithms described
in~\textsection~\ref{sec:dic_learning}.
Before delving into experiments, we discuss how videos and image-sets can
be modeled by linear subspaces and hence
as points on Grassmann manifolds.

 
\subsection{Representing Image-Sets and Videos on Grassmann Manifolds}
\label{sec:image_video_modeling}

Let us define a video as an ordered collection of images with time-stamp
information,
and an image-set as simply an orderless collection of images.
In this section we briefly demonstrate how videos and image-sets can be
modeled by subspaces
(and hence as points on Grassmann manifolds).
We first consider an approach where the time-stamp information is ignored,
followed by an approach where the dynamics of image sequences are taken
into account.

\subsubsection{Modeling of Appearance}
\label{sec:modelling_image_sets}

The appearance of an image-set or video $\mathbb{F} =
\{\Vec{f}_1,\Vec{f}_2,\cdots,\Vec{f}_\tau\}$,
where $\Vec{f}_i \in \R^d$ is the vectorized representation of
$i$-th observation (frame in video),
can be represented by a linear subspace through any orthogonalization
procedure like SVD.
More specifically, let $\Mat{U} \Mat{\Sigma} \Mat{V}^T$
be the SVD of $\mathbb{F}$. The first $p$ columns of $\Mat{U}$ represent
an optimized subspace of order $p$ (in the mean square sense)
for $\mathbb{F}$ and can be seen as a point on the Grassmann manifold
$\GRASS{p}{d}$.
%

Modeling by linear subspaces generally does not take into account the order
of images.
While this property sounds restrictive,
in many practical situations (like object recognition from video),
the order of frames may not be important for decision making.
However, it is possible to capture information related to order through an
extended type of image-sets,
obtained through a block Hankel matrix formalism~\cite{DSA:CVPR:2011}.

\subsubsection{Modeling of Dynamics}
\label{sec:modelling_dynamics}

A video can be represented by an ARMA model to explicitly capture
dynamics~\cite{Doretto_2003,Turaga_PAMI_2011}.
A set of ordered images $\{\Vec{f}(t)\}_{t=1}^{\tau}; \Vec{f}(t) \in
\R^d$
can be modeled as the output of an ARMA model by:

\noindent
\begin{align}
    ~ \Vec{f}(t) & = \Mat{C}\Vec{z}(t)+\Vec{w}(t),  ~~~ \Vec{w}(t)\thicksim
\mathcal{N}(0,\Mat{R}).
      \label{eqn:ARMA_equ1}
      \\
    ~ \Vec{z}(t+1) & = \Mat{A}\Vec{z}(t)+\Vec{v}(t), ~~~~~
\Vec{v}(t)\thicksim \mathcal{N}(0,\Mat{Q}),
      \label{eqn:ARMA_equ2}
\end{align}%

\noindent
where
$\Vec{z}(t) \in \R^n$ is the hidden state vector at time $t$,
$\Mat{A} \in  \R^{n \times n}$
and
$\Mat{C} \in  \R^{d \times n}$
are the transition and measurement matrices, respectively,
while
$\Vec{w}$ and $\Vec{v}$ are noise components modeled as normal distributions
with zero mean and covariance matrices
$\Mat{R} \in  \R^{d \times d}$
and
$\Mat{Q} \in  \R^{n \times n}$,
respectively. Loosely speaking, one advantage of the ARMA model
is that it decouples the appearance of the spatio-temporal data
(modeled by $\Mat{C}$) from the dynamics (represented by $\Mat{A}$).

The transition and measurement matrices can be estimated through a set of
feature vectors.
More specifically, if $\mathbb{F}_\tau=[\Vec{f}(1)| \Vec{f}(2)| \cdots|
\Vec{f}(\tau)]$
represents the feature matrix for time indexes \mbox{$1,2,\cdots,\tau$},
the estimated transition $\Mat{\widehat{A}}$ and measurement
$\Mat{\widehat{C}}$ matrices
can be obtained via the SVD of $\mathbb{F}_\tau=\Mat{U}\Mat{\Sigma}\Mat{V}^T$,
as follows:

\begin{align}
    \Mat{\widehat{A}} & = \Mat{\Sigma}\Mat{V}^T\Mat{D}_1\Mat{V}(\Mat{V}^T
\Mat{D}_2\Mat{V})^{-1}\Mat{\Sigma}^{-1}\;.
    \label{eqn:ARMA_equ3}
    \\
    \Mat{\widehat{C}} & = \Mat{U}\;,
    \label{eqn:ARMA_equ4}
\end{align}%

\noindent
where

\begin{equation*}
\Mat{D}_1 = \begin{bmatrix}
       \Vec{0}_{\tau\mbox{-}1}^T &0           \\
       \mathbf{I}_{\tau - 1} &\Vec{0}_{\tau\mbox{-}1}
     \end{bmatrix}
\mbox{~~and~~}
 \Mat{D}_2 = \begin{bmatrix}
       \mathbf{I}_{\tau-1} &\Vec{0}_{\tau\mbox{-}1}\\
       \Vec{0}_{\tau\mbox{-}1}^T &0
     \end{bmatrix}.
\end{equation*}%

Two ARMA models can be compared based on the subspace angles between the
column-spaces of their observability matrices~\cite{DeCock_2002}.
The extended observability matrix of an ARMA model is given by
\begin{equation*}
\Mat{O}_\infty = [~ \Mat{C}^T|\Mat{(CA)}^T|\Mat{(CA^2)}^T| ~\cdots~
|\Mat{(CA^n)}^T|\cdots ~]^T.
\end{equation*}
The extended observability matrix is usually approximated by the finite
observability matrix as~\cite{Turaga_PAMI_2011}:

\begin{equation}
\Mat{O}_m
=
\left[
  \Mat{C}^T|\Mat{(C A)}^T|\Mat{(C A^2)}^T|\cdots|\Mat{(C A^{(m-1)})}^T
\right]^T.
\label{eqn:ARMA_obs}
\end{equation}%
%

For a given video,
the finite observability parameter of the ARMA model is estimated as
described above.
To represent the subspace spanned by the columns of $\Mat{O}_m$, 
an orthonormal basis can be computed through Gram-Schmidt orthonormalization.
As a result, a linear dynamic system can be described as a point on a
Grassmann manifold
corresponding to the column space of the observability matrix.
The appearance modeling presented in \textsection~
\ref{sec:modelling_image_sets}
can be seen as a special case of ARMA modeling, where {$m = 1$}.


\subsection{Coding on Grassmann Manifolds}

In this part we compare and contrast the performance of the proposed methods
against several state-of-the-art methods:
Discriminant Canonical Correlation Analysis (DCC)~\cite{DCCA:PAMI:2007},
kernelized Affine Hull Method (KAHM)~\cite{Cevikalp_CVPR_2010},
Grassmann Discriminant Analysis (GDA)~\cite{HAMM2008_ICML},
Graph-embedding Grassmann Discriminant Analysis (GGDA)~\cite{Harandi_CVPR_2011}
and the intrinsic sparse coding (iSC)~\cite{Vemuri_ICML_2013}.
We evaluate the performance on the tasks of 
\textbf{(i)} gender recognition from gait, 
\textbf{(ii)} hand gesture recognition,
and 
\textbf{(iii)} scene analysis.

DCC is an iterative learning method that maximizes a measure of
discrimination between image sets
where the distance between sets is expressed by canonical correlations. 
In KAHM, images are considered as points in a linear or affine feature space,
while image sets are characterized by a convex geometric region (affine or
convex hull) spanned by their feature points.
GDA can be considered as an extension of kernel discriminant analysis over
Grassmann manifolds \cite{HAMM2008_ICML}.
In GDA, a transform over the Grassmann manifold is learned
to simultaneously maximize a measure of inter-class distances and minimize
intra-class distances.
GGDA can be considered as an extension of GDA,
where a local discriminant transform over Grassmann manifolds is learned.
This is achieved by incorporating local similarities/dissimilarities
through within-class and between-class similarity graphs.

We denote Grassmann sparse coding, Grassmann locality linear coding and
their kernel extensions by
{\it gSC}, {\it gLC}, {\it kgSC} and {\it kgLC}, respectively.
Based on preliminary experiments, the Gaussian
kernel~\cite{Shawe-Taylor:2004:KMP},
defined as  $k(\Vec{a},\Vec{b}) = \exp\left(-\gamma\|\Vec{a}-\Vec{b}\|^2
\right)$,
was used in kgSC and kgLC.
The value of the $\gamma$ parameter in all experiments was determined by
cross validation.

In the following three experiments, the Grassmannian dictionary for the
proposed gSC, gLC, kgSC and kgLC was constituted of all available training
data. The classification method described in
\textsection~\ref{sec:sparse_classification} was used to determine the
label of a query sample. The class specific residual error in the case of
kgSC and kgLC is obtained by kernelizing 
Eq.~\eqref{eqn:sparse_classification} as:
\begin{equation*}
    \varepsilon_c(\mathcal{X})= \Big\| \mh \Psi_\Mat{X} -
\sum\limits_{j=1}^{N} \sk y_j 
    \mh \Psi_{\Mat{D}_j} \delta\big( l(j)-c \big) \Big\|_F^2\;,
\end{equation*}%
where $l(j)$ is the class label of the $j^\mathrm{th}$ atom and $\delta(x)$
is the discrete Dirac function.


\subsubsection{Gender Recognition from Gait}
\label{exp_CASIA}

Gait is defined as ``manner of walking'' and can be used as a biometric measure
to recognize, among other things, the gender of humans~\cite{GEI_TIP_2009}.
For the task of gender recognition from gait data, we
have used Dataset-B of the CASIA Gait Database~\cite{CASIA_DATASET}
which constitutes of 124 individuals (93 males and 31 females).
In the CASIA dataset, the gait of each subject has been captured from 11
angles.
Every video is represented by one gait energy image (GEI) of size $32
\times 32$,
which has been shown to be effective in recognition of
gender~\cite{GEI_TIP_2009}.
Cropped samples of GEI images are shown in Fig.~\ref{fig:CASIA_example}.

We used the videos captured with normal clothes
and created a subspace of order 6 (based on preliminary experiments) using
the corresponding 11 GEIs.
This resulted in 731  points on  $\GRASS{1024}{6}$.
We then randomly selected 20 individuals (10 male, 10 female) as the
training set and used the remaining individuals for testing.
There is no overlap of individuals between the training and test sets.

Table~\ref{tab:results_CASIA} shows a comparison of gSC, gLC and their
kernelized versions
against DCC, KAHM, GDA and GGDA.
All four proposed methods consistently outperform previous state-of-the-art
algorithms with a big margin. 
The highest accuracy is attained by kgSC, followed by kgLC. As expected,
the kernel extensions perform better than
gSC and gLC. However, the burden of determining the kernel parameters could
be sometimes  overwhelming.

\def\CASIASIZE {0.175}
\begin{figure}[!tb]
      \begin{minipage}{1\columnwidth}
      \centering
        \includegraphics[width=\CASIASIZE
\columnwidth,keepaspectratio]{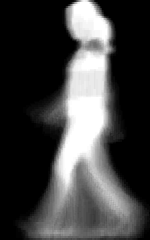}
        \includegraphics[width=\CASIASIZE
\columnwidth,keepaspectratio]{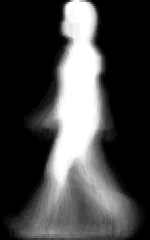}
        \includegraphics[width=\CASIASIZE
\columnwidth,keepaspectratio]{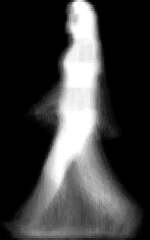}
        \includegraphics[width=\CASIASIZE
\columnwidth,keepaspectratio]{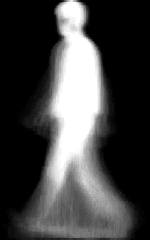}
      \end{minipage}
      \vspace{0.5ex}
  \caption
    {
    \small
    GEI samples from the CASIA gait dataset~\cite{CASIA_DATASET}.
    }
  \label{fig:CASIA_example}
\end{figure}

\begin{table}[!tb]
  \centering
  \caption
    {
    \small
    Recognition rate on the CASIA dataset for
    KAHM~\cite{Cevikalp_CVPR_2010},
    GDA~\cite{HAMM2008_ICML},
    GGDA~\cite{Harandi_CVPR_2011},
    iSC~\cite{Vemuri_ICML_2013}
    and the proposed approaches.
    }
  \label{tab:results_CASIA}
  \begin{tabular}{lc}
    \toprule
    {\bf Method}   & {\bf ~~Accuracy~~}\\
    \toprule
    {DCC~\cite{DCCA:PAMI:2007}}        	 &$85.9 \pm 6.6$\\
    {KAHM~\cite{Cevikalp_CVPR_2010} }        &$89.8 \pm 2.4$\\
    {GDA~\cite{HAMM2008_ICML}}               &$76.4 \pm 5.8$\\
    {GGDA~\cite{Harandi_CVPR_2011}}          &$84.3 \pm 4.8$\\
    {iSC~\cite{Vemuri_ICML_2013}}			 &$86.9 \pm 3.2$\\
    \midrule
    {\bf gSC}                                &{$\bf{94.3 \pm 2.1}$}\\
    {\bf gLC}                               &{$\bf{93.7 \pm 2.1}$}\\
    {\bf kgSC}       		                 &{$\bf{95.6 \pm 2.1}$}\\
    {\bf kgLC}		                         &{$\bf{95.2 \pm 1.6}$}\\
    \bottomrule
  \end{tabular}
\end{table}


\subsubsection{Hand Gesture Recognition}
\label{sec:hand_gesture}

For the hand-gesture recognition task,
we used the Cambridge hand-gesture dataset \cite{Kim_PAMI_2009}
which consists of 900 image sequences of 9 gesture classes.
Each class has 100 image sequences performed by 2 subjects,
captured under 5 illuminations and 10 arbitrary motions.
The 9 classes are defined by three  primitive hand shapes and three
primitive motions.
Each sequence was recorded at 30 fps with a resolution of {\small $320
\times 240$},
in front of a fixed camera having roughly isolated gestures in space and time.
See Fig.~\ref{fig:cambridge_hand_database} for examples.
We followed the test protocol defined by~\citet{Kim_PAMI_2009},
and resized all sequences to {\small $20\times20\times20$}.
Sequences with normal illumination are considered for training
while the remaining sequences (with different illumination characteristics)
are used for testing.

As per~\citet{Kim_PAMI_2009}, we report the recognition rates for the four
illumination sets.
In addition to GDA, GGDA and KAHM, the proposed methods were also compared
against
Tensor Canonical Correlation Analysis (TCCA)~\cite{Kim_PAMI_2009}
and Product Manifolds (PM)~\cite{Lui_JMLR_2012}.
TCCA, as the name implies, is the extension of canonical correlation
analysis to multiway data arrays or tensors.
Canonical correlation analysis is a standard method for measuring the
similarity between subspaces~\cite{Kim_PAMI_2009}.
In the PM method a tensor is characterized as a point on a product manifold
and  classification is performed on this space.
The product manifold is created by applying a modified high order singular
value decomposition on the tensors and interpreting each factorized space
as a Grassmann manifold.

For Grassmann-based methods, we represented each video through ARMA modeling.
The observability order of the ARMA model ($m$ in Eq.~\eqref{eqn:ARMA_obs})
and the subspace dimension (order of matrix $\Mat{C}$) were selected as 5
and 10, respectively.
The results, presented in Table~\ref{tab:table_hand_classification},
show that the proposed approaches obtain the highest performance.
kgLC achieves the best recognition accuracy on all four sets.
KAHM performs very poorly in this task,
which we conjecture is due to the illumination differences between the
training and test sets.

\begin{figure}[!tb]
  \begin{minipage}{1\columnwidth}
  \centering
    \includegraphics[width=0.48\textwidth,keepaspectratio]
{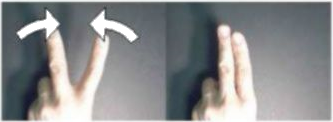}
    \hfill
    \includegraphics[width=0.48\textwidth,keepaspectratio]
{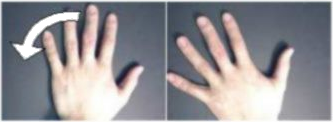}
  \end{minipage}
  \caption
    {
    \small
    Examples of hand actions in the Cambridge dataset~\cite{Kim_PAMI_2009}.
    }
  \label{fig:cambridge_hand_database}
\end{figure}

\begin{table*}[!tb]
  \centering 
  \caption{{\bf Hand-gesture recognition.} Recognition accuracy for the
hand-gesture recognition task using
    KAHM~\cite{Cevikalp_CVPR_2010}, GDA~\cite{HAMM2008_ICML},
GGDA~\cite{Harandi_CVPR_2011}, TCCA~\cite{Kim_PAMI_2009}, 
    Product Manifold~(PM)~\cite{Lui_JMLR_2012}, iSC~\cite{Vemuri_ICML_2013}
and the proposed approaches.      
  } 
  \begin{tabular}{lccccc}
    \toprule
    \bf{Method}    &\bf{Set1}  &\bf{Set2}  &\bf{Set3}  &\bf{Set4} 
&\bf{Overall}\\
    \toprule
    {TCCA~\cite{Kim_PAMI_2009}}            &$81$     &$81$     &$78$   
&$86$   &$82 \pm 3.5$\\
    {KAHM~\cite{Cevikalp_CVPR_2010}}       &$43$     &$43$     &$43$   
&$41$   &$43 \pm 1.4$\\    
    {GDA~\cite{HAMM2008_ICML}}             &$92$     &$85$     &$84$   
&$87$   &$87.4 \pm 3.8$\\
    {GGDA~\cite{Harandi_CVPR_2011}}        &$91$     &$91$     &$88$   
&$94$   &$91.1 \pm 2.5$\\
    {PM~\cite{Lui_JMLR_2012}}              &$93$     &$89$     &$91$   
&$94$   &$91.7 \pm 2.3$\\
    {iSC~\cite{Vemuri_ICML_2013}}		   &$93$	 &$94$	   &$89$	&$92$	&$92.4
\pm 2.1$\\
    \midrule
    {\bf gSC}                  & $\bf{93}$ & $\bf{92}$ & $\bf{93}$ &
$\bf{94}$  & $\bf{93.3 \pm 0.9}$ \\
    {\bf gLC}                 & $\bf{96}$ & $\bf{94}$ & $\bf{96}$ &
$\bf{97}$  & $\bf{95.4 \pm 1.3}$ \\
    {\bf kgSC}	               & $\bf{96}$ & $\bf{92}$ & $\bf{93}$ &
$\bf{97}$  & $\bf{94.4 \pm 2.0}$ \\
    {\bf kgLC}		           & $\bf{96}$ & $\bf{94}$ & $\bf{96}$ & $\bf{98}$ 
& $\bf{95.7 \pm 1.6}$ \\
    \bottomrule
  \end{tabular}  
  \label{tab:table_hand_classification}
\end{table*}


\subsubsection{Scene Analysis}
\label{sec:scene_analysis}

For scene analysis,
we employed the UCSD traffic dataset~\cite{Chan_CVPR2005},
which contains 254 video sequences of highway traffic of varying patterns
(\eg light, heavy)
in various weather conditions (\eg, cloudy, raining, sunny).
Each video was recorded with a resolution of {$320 \times 240$} pixels,
for a duration ranging from 42 to 52 frames.
Here we have used a normalized grayscale $48 \times 48$ version of the dataset.
The normalization process for each video clip involves subtracting the mean
image
and normalizing the pixel intensities to unit variance.
This is useful to reduce the impact of illumination variations.

The dataset is labeled into three classes with respect to the amount of
traffic congestion in each sequence.
In total there are 44 sequences of heavy traffic (slow or stop-and-go speeds),
45 of medium traffic (reduced speed),
and 165 of light traffic (normal speed).
See Fig.~\ref{fig:Traffic_example} for examples.

We represented each video on a Grassmann manifold through ARMA modeling.
The observability order of the ARMA model and the subspace dimension were
selected as 5 and 10 respectively.

In addition to GDA, GGDA and KAHM, the proposed methods were also compared
against Linear Dynamical System (LDS)
and Compressive Sensing Linear Dynamical System
(CS-LDS)~\cite{Turaga_ECCV_2010}.
The results, presented in Table~\ref{tab:table_scene_classification},
show that the proposed approaches obtain the best overall performance, with
kgLC
achieving the highest overall accuracy. 
{
It is worth mentioning that the performance of kgLC competes with the
state-of-the-art algorithms on this 
dataset (\eg, Ravichandran \etal report an accuracy of 95.6\%~\cite{Ravichandran_ACCV_2011}).
}

\begin{figure*}[!tb]
  \begin{minipage}{1\textwidth}
    \centering
    \includegraphics[width=0.32\textwidth]{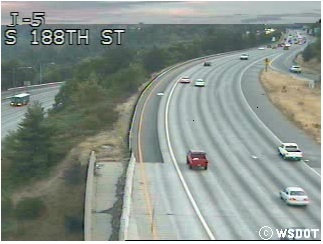}
    \hfill
    \includegraphics[width=0.32\textwidth]{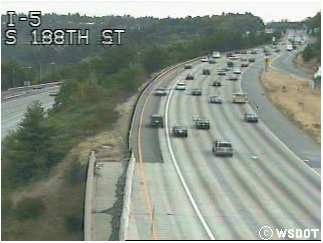}
    \hfill
    \includegraphics[width=0.32\textwidth]{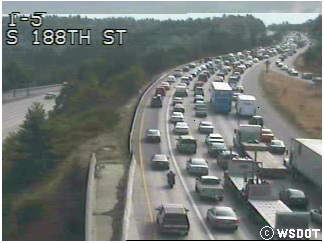}
    \vspace{-2ex}
    \caption
      {
      \small
      Representative examples of the three classes in UCSD traffic video
dataset~\cite{Chan_CVPR2005}.
      From left to right: examples of light, medium, and heavy traffic.
      }
    \label{fig:Traffic_example}
  \end{minipage}
\end{figure*}

\begin{table*}[!tb]
  \centering
  \caption
    {
    Average correct recognition rate on the UCSD video traffic dataset for
dynamic spatio-temporal models using
    LDS~\cite{Turaga_ECCV_2010}, Compressive-Sensing
LDS~\cite{Turaga_ECCV_2010},
    GDA~\cite{HAMM2008_ICML},
    GGDA~\cite{Harandi_CVPR_2011},
    and the proposed approaches.
    }
  \label{tab:table_scene_classification}
  \vspace{1ex}

  \begin{tabular}{lccccc}
    \toprule
    \bf{Method}    &\bf{Exp1}  &\bf{Exp2}  &\bf{Exp3}  &\bf{Exp4} 
&\bf{Overall}\\
    \toprule
    {LDS~\cite{Turaga_ECCV_2010}}         &$85.7$       &$85.9$      
&$87.5$       &$92.1$       &$87.8 \pm 3.0$\\
    {CS-LDS~\cite{Turaga_ECCV_2010}}      &$84.1$       &$87.5$      
&$89.1$       &$85.7$       &$86.6 \pm 2.2$\\
    {KAHM~\cite{Cevikalp_CVPR_2010}}      &$84.1$       &$79.7$      
&$82.8$       &$84.1$       &$82.7 \pm 2.1$\\
    {GDA~\cite{HAMM2008_ICML}}            &$82.5$       &$85.9$      
&$70.3$       &$77.8$       &$79.1 \pm 6.7$\\
    {GGDA~\cite{Harandi_CVPR_2011}}       &$87.3$       &$89.1$      
&$90.6$  		&$90.5$       &$89.4 \pm 1.5$\\
    {iSC~\cite{Vemuri_ICML_2013}}		  &$93.7$	 	&$87.5$	   	 
&$90.6$		&$96.8$		  &$92.2 \pm 4.0$\\
    \midrule
    {\bf gSC}                  			  & $\bf{93.7}$ & $\bf{87.5}$ &
$\bf{95.3}$ & $\bf{95.2}$  & $\bf{92.9 \pm 3.7}$ \\
    {\bf gLC}                 		      & $\bf{96.8}$ & $\bf{85.9}$ &
$\bf{92.2}$ & $\bf{93.7}$  & $\bf{92.2 \pm 4.5}$ \\
    {\bf kgSC}	               			  & $\bf{96.8}$ & $\bf{89.1}$ &
$\bf{95.3}$ & $\bf{98.4}$  & $\bf{94.9 \pm 4.1}$ \\
    {\bf kgLC}		           			  & $\bf{95.2}$ & $\bf{92.2}$ & $\bf{96.9}$ &
$\bf{96.8}$  & $\bf{95.3 \pm 2.2}$ \\
    \bottomrule
  \end{tabular}
\end{table*}

\subsection{Synthetic Data}
\label{sec:synthetic_data}

To contrast the Log-Euclidean (lE-SC) and intrinsic
(iSC)~\cite{Cetingul_TMI_2014,Vemuri_ICML_2013} 
solutions against the proposed gSC approach we performed two experiments
with synthetic data.
Specifically, we considered two multi-class classification problems over
{$\GRASS{2}{6}$}.
The first experiment involved a relatively simple classification problem
that matched the properties of the Log-Euclidean approach,
while the second experiment considered a more realistic scenario.

In both experiments,  we randomly generated four classes over the
{$\GRASS{2}{6}$},
where the samples in each class obey a normal distribution on a specific
tangent space of {$\GRASS{2}{6}$}.
This can be achieved by considering normal distributions over the specific
tangent space of {$\GRASS{2}{6}$}
followed by mapping the points back to {$\GRASS{2}{6}$} using the
exponential map
(see~\citet{Begelfor_CVPR_2006} for details of the exponential map).
We created four classification problems with increasing difficulty,
by fixing the mean of each class and increasing the class variance.
In the following discussion, the problems will be referred as
{\it `easy'}, {\it `medium'}, {\it `hard'}, and {\it `very hard'}.

For a given problem, 8~samples per class were considered as the dictionary
atoms,
while 1000 samples per class were generated as query data.
This results in a multiclass recognition problem with 4000 samples and a
dictionary of size 32.
All the generated samples were then mapped back to the manifold using the
exponential map
and were used in the Log-Euclidean, intrinsic and the proposed sparse
coding approaches.
For each task, the data generation procedure was repeated ten times;
average recognition rates are reported.

In the first experiment, we considered distributions over the identity
tangent space, \ie,
\mbox{$\mathcal{P} = \mathrm{span\Big(}
\begin{bmatrix}
\mathbf{I}_{2}\\
\Mat{0}
\end{bmatrix}
\Big)
$}%
.
The results are presented in Table~\ref{tab:table_syn_data} under
Experiment~$\#1$.
By increasing the class variance, samples from various classes are intertwined,
which in turn leads to a decrease in recognition accuracy.
As for the log-Euclidean approach, we considered two setups. In the first
setup, the center of 
projection was fixed at $\mathcal{P}$. In the second setup (shown as
{lE-SC-adaptive}), the center 
of projection was set to the Fr\'{e}chet mean of data. 
Even though this experiment matches the characteristics of the
Log-Euclidean approach
(since the prior knowledge of class distribution is available),
both gSC and iSC approaches obtain on par performance for the easy case.
For the medium case, Log-Euclidean approaches achieve higher accuracy
followed by gSC and iSC.
We note that the fixed log-Euclidean method performs better than the
adaptive setup for this experiment.

In the second experiment we relaxed the location of tangent space
in order to simulate a more challenging scenario.
More specifically, instead of generating distributions over the identity
tangent space,
the tangent space was selected randomly.
As shown in Table~\ref{tab:table_syn_data} under Experiment~$\#2$,
the Log-Euclidean approaches perform poorly when compared to gSC and iSC.
Among the two setups of the log-Euclidean approach, the adaptive one
performs better than the fixed one.
Similar to the previous experiments, the gSC approach consistently
outperforms iSC.

\begin{table}[!tb]
  \centering
  \caption
    {
    \small
    Comparison of the proposed gSC approach with the Log-Euclidean sparse
coding (lE-SC)
    and intrinsic sparse coding (iSC)~\cite{Vemuri_ICML_2013} methods on
synthetic data.
    In the first experiment, samples in each class obey a normal
distribution over the identity tangent space.
    The second experiment reflects a more challenging scenario where
samples in each class obey a normal distribution over a random 
    tangent space, instead of the identity tangent space.
    }
  \label{tab:table_syn_data}
  \begin{tabular}{l cc cc}
    \toprule
    &\multicolumn{2}{c}{Experiment~$\#$1}
&\multicolumn{2}{c}{Experiment~$\#2$}\\
    \bf{Task}    &\bf{Easy}  &\bf{Medium}  &\bf{Hard}  &\bf{Very Hard}\\
    \midrule
    {lE-SC}             				& $99.5\%$      & $90.8\%$      & $55.7\%$     
& $49.9\%$             \\
    {lE-SC-adaptive}       				& $92.9\%$      & $88.7\%$      & $57.3\%$  
   & $50.6\%$             \\
    {iSC~\cite{Vemuri_ICML_2013}}       & $98.6\%$      & $84.1\%$      &
$64.7\%$      & $53.4\%$             \\
    {gSC}             				    & $99.2\%$      & $86.6\%$      & $66.9\%$   
  & $57.4\%$             \\
    \bottomrule
  \end{tabular}
\end{table}

\subsection{Dictionary Learning}

Here we analyze the performance of the proposed dictionary learning
techniques as described in \textsection~\ref{sec:dic_learning} 
on three classification tasks: face recognition, action recognition and
dynamic texture classification. 
In all the following experiments, an SVM classifier with a Gaussian kernel
was used to perform recognition. 
That is, the training and testing data were first coded by the learned
dictionary and then the sparse codes were fed to an SVM classifier.
Parameters for the SVM classifier were determined by cross validation.
 
Since the intrinsic dictionary learning as proposed
by~\citet{Vemuri_ICML_2013} 
has no analytic solution on Grassmann manifolds, we will just compare gSC,
gLC and their 
kernel extensions in conjunction with dictionary learning against 
Discriminant Canonical Correlation Analysis (DCC)~\cite{DCCA:PAMI:2007},
kernelized Affine Hull Method (KAHM)~\cite{Cevikalp_CVPR_2010},
Grassmann Discriminant Analysis (GDA)~\cite{HAMM2008_ICML} and 
Graph-embedding Grassmann Discriminant Analysis (GGDA)~\cite{Harandi_CVPR_2011}
in the following experiments.

\subsubsection{Face Recognition}
\label{sec:face_recognition}

While face recognition from a single still image has been extensively studied,
recognition based on a group of still images is relatively new.
A popular choice for modeling image-sets
is by representing them through linear subspaces
\cite{HAMM2008_ICML,Harandi_CVPR_2011}.
For the task of image-set face recognition, we used the YouTube celebrity
dataset \cite{YT_Celebrity}
which contains 1910 video clips of 47 subjects.
See Fig.~\ref{fig:YT_Celebrity_example} for examples.
Face recognition on this dataset is challenging, 
since the videos have a high compression ratio and most of them have
low-resolution.

To create an image set from a video, we used a cascaded face locator
\cite{Viola:IJCV:2004} 
to extract face regions from each video, followed by
resizing regions to $96 \times 96$ and describing them via histogram of
Local Binary Patterns 
(LBP)~\cite{LBP_PAMI_2002}.
Then each image set (corresponding to a video) was represented by a linear
subspace of order $5$. 
We randomly chose 70\% of the dataset for training and the remaining 30\%
for testing.
The process of random splitting was repeated ten times and the average
classification accuracy is reported. 

The results in Table \ref{tab:table_face_rec} show that the proposed
coding methods (using dictionaries provided by their corresponding
dictionary learning algorithms)
outperform the competitors.
kgSC with dictionary learning achieved the highest accuracy of 73.91\%, more
than 3 percentage points better than gSC with dictionary learning.
Similarly, the performance of kgLC with dictionary learning is observed to
be higher than gLC with dictionary learning.

\begin{table}[!b]
  \centering
  \caption
    {
    Average correct recognition rate (CRR) on the YouTube celebrity dataset.
    }
  \label{tab:table_face_rec}
%
  \begin{tabular}{lccccc}
    \toprule
    \bf{Method}    &\bf{CRR}  \\
    \toprule
    {DCC~\cite{DCCA:PAMI:2007}}      			  &$60.21 \pm 2.9$ \\
    {KAHM~\cite{Cevikalp_CVPR_2010}}              &$67.49 \pm 3.5$ \\
    {GDA~\cite{HAMM2008_ICML}}                    &$58.72 \pm 3.0$\\
    {GGDA~\cite{Harandi_CVPR_2011}}               &$61.06 \pm 2.2$\\
    \midrule
    {\bf{gSC-dic}}                                &$\bf{70.47 \pm 1.7}$\\
    {\bf{gLC-dic}}                               &$\bf{71.74 \pm 2.3}$\\
    {\bf{kgSC-dic}}                               &$\bf{73.91 \pm 1.9}$\\
    {\bf{kgLC-dic}}                              &$\bf{73.53 \pm 2.3}$\\
    \bottomrule
  \end{tabular}
\end{table}

\def \YTSIZE {0.225}
\begin{figure}[!tb]  
  \begin{minipage}{1\columnwidth}\center
  \includegraphics[width=\YTSIZE \columnwidth,keepaspectratio]{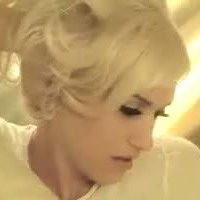}
  \includegraphics[width=\YTSIZE \columnwidth,keepaspectratio]{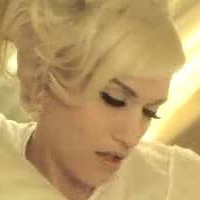}
  \includegraphics[width=\YTSIZE \columnwidth,keepaspectratio]{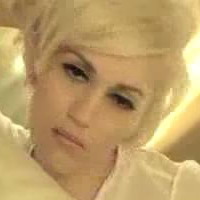}
  \includegraphics[width=\YTSIZE \columnwidth,keepaspectratio]{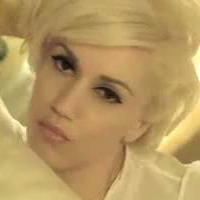}
  \end{minipage}
  \begin{minipage}{1\columnwidth}\center
  \includegraphics[width=\YTSIZE \columnwidth,keepaspectratio]{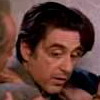}
  \includegraphics[width=\YTSIZE \columnwidth,keepaspectratio]{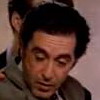}
  \includegraphics[width=\YTSIZE \columnwidth,keepaspectratio]{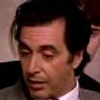}
  \includegraphics[width=\YTSIZE \columnwidth,keepaspectratio]{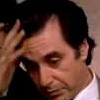}
  \end{minipage}
  \caption
    {
    Examples of YouTube celebrity dataset (grayscale versions of images
were used in our experiments).
    }
  \label{fig:YT_Celebrity_example}
\end{figure}


\subsubsection{Ballet Dataset}
\label{exp_Ballet}

The Ballet dataset contains 44 videos collected from an instructional
ballet DVD \cite{Ballet_Dataset}.
The dataset consists of 8 complex motion patterns performed by 3 subjects,
The actions include:
{\it `left-to-right hand opening'},
{\it `right-to-left hand opening'},
{\it `standing hand opening'},
{\it `leg swinging'},
{\it `jumping'},
{\it `turning'},
{\it `hopping'} and
{\it `standing still'}.
Fig.~\ref{fig:Ballet_example} shows examples.
The dataset is challenging due to the significant intra-class variations in
terms of speed,
spatial and temporal scale, clothing and movement.

We extracted 2400 image sets by grouping 6 frames that exhibited the same
action into one image set.
We described each image set by a subspace of order~4 with Histogram of
Oriented Gradients (HOG) as frame descriptor~\cite{Dalal:2005:HOG}.
Available samples were randomly split into training and testing sets
(the number of image sets in both sets was even).
The process of random splitting was repeated
ten times and the average classification accuracy is reported.

Table~\ref{tab:table_ballet} shows that all proposed coding approaches
have superior performance as compared to DCC, KAHM, GDA and GGDA.
For example, the difference between gSC with dictionary learning (gSC-dic)
and the closest state-of-the-art competitor (GGDA), is more than six
percentage points.

\begin{table}[!b]
  \centering
  \caption
    {
    Average recognition rate on the Ballet dataset.
    }
  \label{tab:table_ballet}
%
  \begin{tabular}{lccccc}
    \toprule
    \bf{Method}    &\bf{CRR}  \\
    \toprule
    {DCC~\cite{DCCA:PAMI:2007}}            	  	  &$41.95 \pm 9.6$ \\
    {KAHM~\cite{Cevikalp_CVPR_2010}}              &$70.05 \pm 0.9$ \\
    {GDA~\cite{HAMM2008_ICML}}                    &$67.33 \pm 1.1$\\
    {GGDA~\cite{Harandi_CVPR_2011}}               &$73.54 \pm 2.0$\\
    \midrule
    {\bf{gSC-dic}}                                &$\bf{79.64 \pm 1.1}$\\
    {\bf{gLC-dic}}                               &$\bf{81.42 \pm 0.8}$\\
    {\bf{kgSC-dic}}                               &$\bf{83.53 \pm 0.8}$\\
    {\bf{kgLC-dic}}                              &$\bf{86.94 \pm 1.1}$\\
    \bottomrule
  \end{tabular}
\end{table}

\def\BALLETSIZE {0.225}
\begin{figure}[!tb]
      \begin{minipage}{1.0\columnwidth}
      \center
        \includegraphics[width=\BALLETSIZE
\columnwidth,keepaspectratio]{ballet_fem1.jpg}
        \includegraphics[width=\BALLETSIZE
\columnwidth,keepaspectratio]{ballet_fem2.jpg}
        \includegraphics[width=\BALLETSIZE
\columnwidth,keepaspectratio]{ballet_fem3.jpg}
        \includegraphics[width=\BALLETSIZE
\columnwidth,keepaspectratio]{ballet_fem4.jpg}
      \end{minipage}
      \begin{minipage}{1.0\columnwidth}
      \center
        \includegraphics[width=\BALLETSIZE
\columnwidth,keepaspectratio]{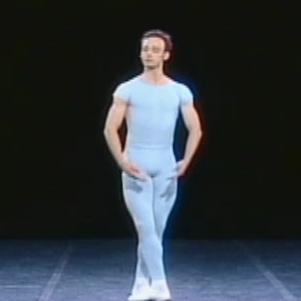}
        \includegraphics[width=\BALLETSIZE
\columnwidth,keepaspectratio]{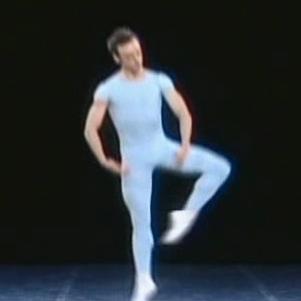}
        \includegraphics[width=\BALLETSIZE
\columnwidth,keepaspectratio]{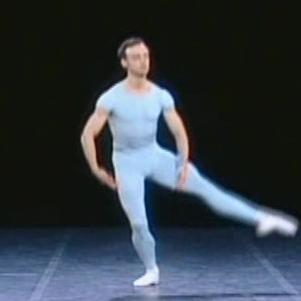}
        \includegraphics[width=\BALLETSIZE
\columnwidth,keepaspectratio]{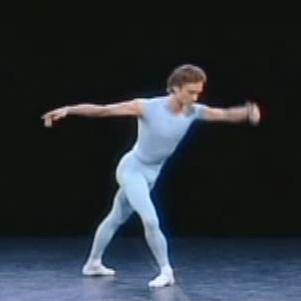}
      \end{minipage}
  \caption
    {
    Examples from the Ballet dataset~\cite{Ballet_Dataset}.
    }
  \label{fig:Ballet_example}
\end{figure}


\subsubsection{Dynamic Texture Classification}
\label{sec:dynamic_texture}

Dynamic textures are videos of moving scenes
that exhibit certain stationary properties in the time
domain~\cite{DynTex_Dataset,XU_ICCV_2011}.
Such videos are pervasive in various environments,
such as sequences of rivers, clouds, fire, swarms of birds, humans in crowds.
In our experiment, we used the challenging DynTex++
dataset~\cite{DynTex_Dataset},
which is comprised of 36 classes,
each of which contains 100 sequences with a fixed size of $50 \times 50
\times 50$
(see Fig.~\ref{fig:DynTex_example} for example classes).
We split the dataset into training and testing sets by randomly assigning
half of the videos of each class to 
the training set and using the rest as query data.
The random split was repeated twenty times; average accuracy is reported.

To generate Grassmann points,
we used histogram of LBP from Three Orthogonal Planes
(LBP-TOP)~\cite{LBPTOP:PAMI:2007}
which takes into account the dynamics within the videos.
To this end, each video is split into subvideos of length 10, with a 7
frame overlap.
Each subvideo is then described by a histogram of LBP-TOP features.
From the subvideo descriptors, we extracted a subspace of order 5 as the
video representation on a Grassmann manifold.  

In addition to DCC, KAHM, GDA and GGDA,
the proposed approaches were compared against two methods specifically
designed for dynamic texture classification:
dynamic fractal spectrum (DFS)~\cite{XU_ICCV_2011}
and Distance Learning Pegasos (DL-Pegasos)~\cite{DynTex_Dataset}.
DFS can be seen as concatenation of two components:
(i)~a volumetric component that encodes the stochastic self-similarities of
dynamic textures as 3D volumes, and 
(ii)~a multi-slice dynamic component that captures structures of dynamic
textures on 2D slices along various views of the 3D volume.
DL-Pegasos uses three descriptors (LBP, HOG and LDS)
and learns how the descriptors can be linearly combined to best
discriminate between dynamic texture classes.

The overall classification results are presented in
Table~\ref{tab:table_dyntex}.
The proposed kgLC with dictionary learning (kgLC-dic) obtains the highest
average recognition rate.

\def \DTSIZE {0.225}
\begin{figure}[!tb]
  \centering
  \begin{minipage}{1\columnwidth}
  \includegraphics[width=\DTSIZE \columnwidth,keepaspectratio]
{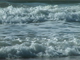}
  \includegraphics[width=\DTSIZE \columnwidth,keepaspectratio]
{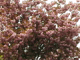}
  \includegraphics[width=\DTSIZE \columnwidth,keepaspectratio]
{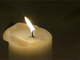}
  \includegraphics[width=\DTSIZE \columnwidth,keepaspectratio]
{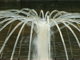}
  \end{minipage}
  \begin{minipage}{1\columnwidth}
  \includegraphics[width=\DTSIZE \columnwidth,keepaspectratio]
{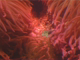}
  \includegraphics[width=\DTSIZE \columnwidth,keepaspectratio]
{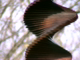}
  \includegraphics[width=\DTSIZE \columnwidth,keepaspectratio]
{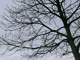}
  \includegraphics[width=\DTSIZE \columnwidth,keepaspectratio]
{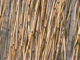}
  \end{minipage}
  \vspace{1ex}
  \caption
    {
    Example classes of DynTex++ dataset (grayscale images were used in our
experiments).
    }
  \label{fig:DynTex_example}
\end{figure}

\begin{table}[!tb]
  \centering
  \caption
    {
    Average recognition rate on the DynTex++ dataset.
    }
  \label{tab:table_dyntex}
%
  \begin{tabular}{lccccc}
    \toprule
    \bf{Method}    &\bf{CRR}  \\
    \toprule
    {DL-PEGASOS~\cite{DynTex_Dataset}}            &$63.7$\\
    {DFS~\cite{XU_ICCV_2011}}                     &$89.9$ \\ 
    {DCC~\cite{DCCA:PAMI:2007}}				      &$53.2$ \\ 
    {KAHM~\cite{Cevikalp_CVPR_2010}}              &$82.8$ \\
    {GDA~\cite{HAMM2008_ICML}}                    &$81.2$\\
    {GGDA~\cite{Harandi_CVPR_2011}}               &$84.1$\\
    \midrule
    {\bf{gSC-dic}}                                    &$\bf{90.3}$\\
    {\bf{gLC-dic}}                                   &$\bf{91.8}$\\
    {\bf{kgSC-dic}}                             &$\bf{92.8}$\\
    {\bf{kgLC-dic}}                            &$\bf{93.2}$\\
    \bottomrule
  \end{tabular}
\end{table}

\subsection{Computational Complexity.}

Let $\mathbb{D} = \big\{ \mathcal{D}_i \big\}_{i=1}^{N},\; \mathcal{D}_i
\in \GRASS{p}{d}$ be a Grassmannian dictionary 
and $\mathcal{X} \in \GRASS{p}{d}$ be a query sample with $\mathcal{D}_j =
\mathrm{span}(\Mat{D}_j)$ and $\mathcal{X} = \mathrm{span}(\Mat{X})$.
In terms of computational load, the gSC algorithm requires the values in
the form of $\|\Mat{X}^T\Mat{D}_j\|_F^2$ which can be computed in
${O}(Ndp^2)$ flops for the whole dictionary. 

The iSC algorithm~\cite{Vemuri_ICML_2013} solves~\eqref{eqn:intrinsic_sc}
for coding. To this end, computing the logarithm map on $\GRASS{p}{d}$ is required.
A very efficient implementation of the logarithm map on $\GRASS{p}{d}$
requires a matrix inversion of size $p \times p$, 
two matrix multiplications of size $d \times p$, and a thin SVD of size $d
\times p$. Computing thin SVD using a stable algorithm like 
the Golub-Reinsch~\cite{Golub:1996:MC} requires $14dp^2 + 8p^3$ flops.
This adds up to a total of ${O}\big(9Np^3 + 16Ndp^2\big)$ flops for the
whole dictionary.

To give the reader a better sense on the computational efficiency of gSC
algorithm, we performed an experiment.
Assuming that the complexity of vector sparse coding for both algorithms is
similar (iSC is a constrained coding 
approach so it is very likely to be more expensive than an unconstrained
one like gSC), we measured the time 
required to compute $\mathcal{K}_\Mat{X}$ in Eq.~\eqref{eqn:Opt_Grass3}
against projecting $\mathbb{D}$
to the tangent space of $\mathcal{X}$. 
To this end,
we considered three cases using the geometry of $\GRASS{3}{100}$,
$\GRASS{3}{1000}$  and $\GRASS{3}{10000}$. 
We randomly generated a dictionary of size 1000 for each case and measured
the time required to compute $\mathcal{K}_\Mat{X}$ and tangent projection
for 1000 query points.
The results given in Table~\ref{tab:table_running_time} show that the gSC
algorithm is significantly faster than iSC.

\begin{table}[!tb]
  \centering
  \caption
    {
    \small
    Running time comparison between the proposed gSC approach and intrinsic 
    sparse coding (iSC)~\cite{Vemuri_ICML_2013} method on synthetic data.
    Times are measured in second on a Quad-core i7 machine with Matlab.
    }
  \label{tab:table_running_time}
  \begin{tabular}{l c c c}
    \toprule
    \bf{Task}    &$\GRASS{3}{100}$ &$\GRASS{3}{1000}$  &$\GRASS{3}{10000}$\\
    \midrule    
    {iSC~\cite{Vemuri_ICML_2013}}       & 77.8s      & 234.2s      & 1320.7s \\
    {gSC}             				    & 4.1s        & 16.9s        & 106.4s    \\
    \bottomrule
  \end{tabular}
\end{table}

\section{Main Findings and Future Directions}
\label{sec:conclusion}

With the aim of coding on Grassmann manifolds,
we proposed to embed such manifolds into the space of symmetric matrices by
an isometric projection.
We then showed how sparse coding and locality linear coding can be
performed in the induced space.
We also tackled the problem of dictionary learning on Grassmann manifolds 
and devised a closed-form solution for updating a dictionary atom by atom,
using the geometry of induced space.
Finally, we proposed a kernelized version of sparse coding, locality linear
coding and dictionary learning on Grassmann manifolds,
to handle non-linearity in data.

Experiments on several classification tasks
(gender recognition, gesture classification, scene analysis, face
recognition, action recognition and dynamic texture
classification) show that the proposed approaches achieve notable
improvements in discrimination accuracy,
in comparison to state-of-the-art methods such as discriminant analysis of
canonical correlation analysis~\cite{DCCA:PAMI:2007}
affine hull method~\cite{Cevikalp_CVPR_2010},
Grassmann discriminant analysis~\cite{HAMM2008_ICML},
graph-embedding Grassmann discriminant analysis~\cite{Harandi_CVPR_2011}
and intrinsic sparse coding method~\cite{Vemuri_ICML_2013}.

In this work a Grassmann dictionary is learned such that a reconstruction
error is minimized.
This is not necessarily the optimum solution when labeled data is available.
To benefit from labeled data, it has recently been proposed to consider a
discriminative penalty term
along with the reconstruction error term in the optimization
process~\cite{Mairal_PAMI12}.
We are currently pursuing this line of research and seeking solutions for
discriminative dictionary learning on Grassmann manifolds.
Moreover, our formulation can be understood as an extrinsic solution to the
problem of coding and dictionary learning on Grassmann manifolds. 
It would be interesting to devise intrinsic solutions based on the geometry
of the induced space, \ie, symmetric matrices.

\section*{Acknowledgements}
\label{sec:Acknowledgements}

NICTA is funded by the Australian Government
as represented by the {\it Department of Broadband, Communications and the
Digital Economy},
as well as the Australian Research Council through the {\it ICT Centre of
Excellence} program.
This work is funded in part through an ARC Discovery grant DP130104567.
C. Shen's participation was in part supported by ARC Future Fellowship
F120100969.


\section{Appendix}

In this appendix, we give proofs for the following theorems.

\begin{theorem}
\label{thm:closest-point-on-Grassman}
Let $\m X$ be an $d \times d$ symmetric matrix
with eigenvalue decomposition $\m X = \m U \m D \m U^T$,
where $\m D$ contains the eigenvalues $\lambda_i$ of $\m X$ in
descending order.  Let $\m U_p$ be the $d \times p$ matrix
consisting of the first $p$ columns of $\m U$. 
Then $\mh U_p = \m U_p \m U_p^T$ is the closest matrix in $\PGRASS{p}{d}$ to $\m X$
(under the Frobenius norm). 
\end{theorem}

\begin{proof}
Observe that
\(
\| \mh V - \m X\|_F^2 = \|\mh V \|_F^2 + \| \m X \|_F^2
-2 \lip \mh V, \m X \rip.
\)
Since $\|\mh V\|_F$ (for $\mh V \in \PGRASS{p}{d}$) and $\|\m X\|_F$ are fixed, minimizing
$\| \mh V - \m X\|_F$ over $\mh V \in \PGRASS{p}{d}$ is the same as maximizing 
$\lip \mh V, \m X \rip$.
If $\mh V = \m V \m V^T$, we may write
\(
\lip \mh V, \m X\rip = \tr ( \m V \m V^T \m X) = \tr(\m V^T \m X \m V),
\)
so it is sufficient to maximize $\tr(\m V^T \m X \m V)$ over
$\m V \in \GRASS{p}{n}$.

If $\m X = \m U \diag(\lambda_1, \ldots, \lambda_d) \m U^T$,
then $\m U_p^T \m X \m U_p = \diag(\lambda_1, \ldots, \lambda_p)$
and $\tr (\m U_p^T \m X \m U_p) = \sum_{i=1}^p \lambda_i$.
On the other hand, let $\m W \in \GRASS{p}{d}$.
Then $\m W^T \m X \m W$ is symmetric of dimension $p\times p$.
Let $\mu_1 \ge \mu_2 \ge \ldots \ge \mu_p$ 
be its eigenvalues and
$\Vec a_i, \, i=1, \ldots, p$ 
the corresponding unit eigenvectors.
Let $\Vec w_i = \m W \Vec a_i$.  Then the $\Vec w_i$ are orthogonal
unit vectors, and $\Vec w_i^T \m X \Vec w_i = \mu_i$.

For $k = 1$ to $p$, let $A_k$ be the subspace of $R^d$ 
spanned by $\Vec w_1, \ldots, \Vec w_k$ and $B_k$ be the space
spanned by the eigenvectors
$\Vec u_k, \ldots, \Vec u_d$ of $\m X$.  Counting dimensions, $A_k$ and $B_k$ 
must have non-trivial intersection.  Let $\Vec v$ be a non-zero vector in this
intersection, and write $\Vec v = \sum_{i=1}^k \alpha_i \Vec w_i= \sum_{i=k}^d \beta_i \Vec u_i$.  Then
\begin{align}
\begin{split}
\label{eq:Courant-inequality}
\mu_k \le
\frac{\sum_{i=1}^k \alpha_i^2 \mu_i}{\sum_{i=1}^k \alpha_i^2 } 
 =
\frac{\Vec v^T \m X \Vec v}{\Vec v^T\Vec v} = 
\frac{\sum_{i=k}^d \beta_i^2 \lambda_i}{\sum_{i=k}^d \beta_i^2 } 
\le \lambda_k ~.
\end{split}
\end{align}
Therefore $\mu_k \le \lambda_k$ and 
\(
\tr (\m W^T \m X \m W) = \sum_{i=1}^p \mu_i \le \sum_{i=1}^p \lambda_i = 
\tr (\m U^T \m X \m U) ~.  
\)
\qed 
\end{proof}


\paragraph{\bf The chordal mean. }
For two points (matrices) $\mh X$ and $\mh Y$ in $\PGRASS{p}{d}$ the
distance $\| \mh X - \mh Y\|_F$ is called the {\em chordal distance}
between the two points.  Given several points $\mh X_i$, 
the {\em $\ell_2$ chordal mean} of $\{\mh X_i\}_{i=1}^m$ is the element $\mh Y \in
\PGRASS{p}{d}$ that minimizes $\sum_{i=1}^m \| \mh Y - \mh X_i\|_F^2$.
There is a closed-form solution for the chordal mean of a set of points
in a Grassman manifold.

\begin{theorem}
\label{thm:chordal-mean}
The chordal mean of a set of points $\mh X_i \in\PGRASS{p}{d}$
is equal to 
\(
{\rm Proj} (\sum_{i=1}^m \mh X_i).
\)
\end{theorem}

\begin{proof}
The proof is analogous to the formula for the chordal mean of 
rotation matrices, given in \cite{Hartley_IJCV_13}.
By the same argument as in Theorem~\ref{thm:closest-point-on-Grassman},
minimizing $\sum_{i=1}^m \|\mh X_i - \mh Y\|_F^2$ is equivalent to
maximizing $\sum_{i=1}^m \lip \mh X_i, \mh Y \rip
= \lip \sum_{i=1}^m \,\mh X_i, \mh Y \rip$.  Thus, the required
$\mh Y$ is the closest point in $\PGRASS{p}{d}$ to 
$\sum_{i=1}^m \,\mh X_i$, as stated. \qed
\end{proof}

	\small 
	\balance
	\bibliographystyle{ieee}	
	\bibliography{references}
\end{document}